%% file: Yang08a.tex
\documentclass[leqno,twoside,11pt]{article}
\usepackage[centertags]{amsmath}
\usepackage{jair}
\usepackage{theapa}
\usepackage{graphicx}
\usepackage{psfig}
\usepackage{subfigure}

\jairheading{32}{2008}{631-662}{11/07}{06/08} \ShortHeadings{A General
Theory of Additive State Space Abstractions} { Yang, Culberson, Holte,
Zahavi \& Felner} \firstpageno{631}

\input{macros.tex}
\begin{document}

\title{A General Theory of Additive State Space Abstractions}

\author{\name Fan Yang \email fyang@cs.ualberta.ca \\
       \name  Joseph Culberson \email  joe@cs.ualberta.ca \\
       \name Robert Holte \email holte@cs.ualberta.ca\\
       \addr Computing Science Department, University of Alberta\\
       Edmonton, Alberta  T6G 2E8 Canada \\
       \name Uzi Zahavi \email zahaviu@cs.biu.ac.il \\
       \addr Computer Science Department, Bar-Ilan University\\
             Ramat-Gan, Israel 92500\\
       \name Ariel Felner \email felner@bgu.ac.il \\
       \addr Information Systems Engineering Department, \\
        Deutsche Telekom Labs, \\
               Ben-Gurion University. \\
              Be'er-Sheva, Israel 85104
}


\maketitle

\begin{abstract}
Informally, a set of abstractions of a state space $S$ is additive if the
distance between any two states in $S$ is always greater than or equal to
the sum of the corresponding distances in the abstract spaces. The first
known additive abstractions, called disjoint pattern databases, were
experimentally demonstrated to produce state of the art performance on
certain state spaces. However, previous applications were restricted to
state spaces with special properties, which precludes disjoint pattern
databases from being defined for several commonly used testbeds, such as
Rubik's Cube, TopSpin and the Pancake puzzle. In this paper we give a
general definition of additive abstractions that can be applied to any
state space and prove that heuristics based on additive abstractions are
consistent as well as admissible. We use this new definition to create
additive abstractions for these testbeds and show experimentally that well
chosen additive abstractions can reduce search time substantially for the
(18,4)-TopSpin puzzle and by three orders of magnitude over state of the
art methods for the 17-Pancake puzzle. We also derive a way of testing if
the heuristic value returned by additive abstractions is provably too low
and show that the use of this test can reduce search time for the
15-puzzle and TopSpin by roughly a factor of two.
\end{abstract}

\section{Introduction}
\label{Introduction}

In its purest form, single-agent heuristic search is concerned with the problem of
finding a least-cost path between two states ({\em start} and {\em goal}) in
a state space given a heuristic function $\htg$ that estimates the
cost to reach the goal state $\goal$ from any state $t$.
Standard algorithms for
single-agent heuristic search such as $\textit{IDA}^*$ \cite{IDAstar} are
guaranteed to find optimal paths if $\htg$ is
{\em admissible},  i.e.\ never
overestimates the actual cost to the goal state from $t$,
and their
efficiency is heavily influenced by the accuracy of $\htg$.
Considerable research has therefore investigated methods for
defining accurate, admissible heuristics.

A common method for defining admissible heuristics,
which has led to major advances in combinatorial
problems \cite{PDB98,sop,rubikPDB,KorfTaylor1996} and planning \cite{planningPDB},
is to ``abstract"
the original state space to create a new, smaller
state space with the key property that for each path $\gpath$\  in the
original space there is a corresponding abstract path whose cost
does not exceed the cost of $\gpath$. Given an abstraction, $\htg$ can be
defined as the cost of the least-cost abstract path from the
abstract state corresponding to $t$ to the abstract state
corresponding to $g$. The best heuristic functions defined by
abstraction are typically based on several abstractions, and are
equal to
either the maximum, or the sum, of the costs returned by
the abstractions \cite{disjointPDB,ADDPDB,maxingAIJ06}.

The sum of the costs returned by a set of abstractions is not always
admissible. If it is, the set of abstractions is said to be
``additive". The main contribution of this paper is to identify
general conditions for abstractions to be additive. The new
conditions subsume most previous notions of ``additive" as special cases.
The greater generality allows additive abstractions to be
defined for state spaces that had no additive abstractions according
to previous definitions, such as Rubik's Cube, TopSpin, the
Pancake puzzle, and related
real-world problems such as the genome rearrangement problem
described by \citeA{genome}.
Our definitions are fully formal, enabling rigorous
proofs of the admissibility and consistency
of the heuristics
defined by our abstractions.
Heuristic
$\htg$ is {\em consistent} if for all states $t$, $g$ and
$u, \htg \le cost(t,u)+h(u,g)$,
where $cost(t,u)$ is the cost of the least-cost path from $t$ to $u$.

The usefulness of our general definitions is
demonstrated experimentally by defining additive abstractions
that substantially reduce the CPU time needed to solve TopSpin
and the Pancake puzzle.  For example, the use of additive abstractions
allows the 17-Pancake puzzle to be solved three orders of magnitude faster than
previous state-of-the-art methods.

Additional experiments show that additive abstractions are not always the
best abstraction method.  The main reason for this is that the solution cost
calculated by an individual additive abstraction can sometimes be
very low.  In the extreme case, which actually arises in practice,
all problems can have abstract solutions
that cost 0.  The final contribution of the paper is to introduce a technique
that is sometimes able to identify that the sum of the costs of the additive
abstractions is provably too small (``infeasible").


The remainder of the paper is organized as follows.
An informal introduction to abstraction
is given in Section \ref{heuristicAbstractions}.
Section \ref{General_Definitions} presents formal general definitions
for abstractions that extend to general additive abstractions.
We provide lemmas proving
the admissibility and consistency of both standard and additive heuristics based on these abstractions.
This section also discusses the relation to previous definitions.
Section \ref{applications} describes successful applications of
additive abstractions to TopSpin and the Pancake puzzle.
Section \ref{negative} discusses the negative results.
Section \ref{infeasible} introduces ``infeasibility"
and presents experimental results showing its effectiveness on
the sliding tile puzzle and TopSpin.
Conclusions are presented in Section \ref{sec-conclusions}.

\section{Heuristics Defined by Abstraction}
\label{heuristicAbstractions}

To illustrate the idea of abstraction and how it is used to define
heuristics, consider the well-known 8-puzzle (the $3 \times 3$
sliding tile puzzle). In this puzzle there are $9$ locations in the
form of a $3 \times 3$ grid and $8$ tiles, numbered 1--8, with the
$9^{th}$ location being empty (or blank). A tile that is adjacent to the
empty location can be moved into the empty location; every move has
a cost of 1. The most common way of abstracting this state space is
to treat several of the tiles as if they were indistinguishable instead of
being distinct \cite{PDB96}. An extreme version of this type of
abstraction is shown in Figure \ref{fig-8puzzle}. Here the tiles are
all indistinguishable from each other, so an abstract state is
entirely defined by the position of the blank. There are therefore
only 9 abstract states, connected as shown in Figure
\ref{fig-8puzzle}. The goal state in the original puzzle has the
blank in the upper left corner, so the abstract goal is the state
shown at the top of the figure. The number beside each abstract
state is the distance from the abstract state to the abstract goal.
For example, in Figure \ref{fig-8puzzle}, abstract state $e$ is 2
moves from the abstract goal. A heuristic function $h(t,g)$ for the distance
from state $t$  to $g$ in the original space is computed in two steps: (1)
compute the abstract state corresponding to $t$ (in this example,
this is done by determining the location of the blank in state $t$);
and then (2) determine the distance from that abstract state to
the abstract goal.
The calculation of the abstract distance can either be done in a preprocessing
step
to create a heuristic lookup table called a {\em pattern database} \cite{PDB94,PDB96}
or at the time it is needed \cite{HierarchicalA,HierarchialRevisited,FelnerAdler}.

\begin{figure}[!ht]
\centerline{
\includegraphics[width=9cm]{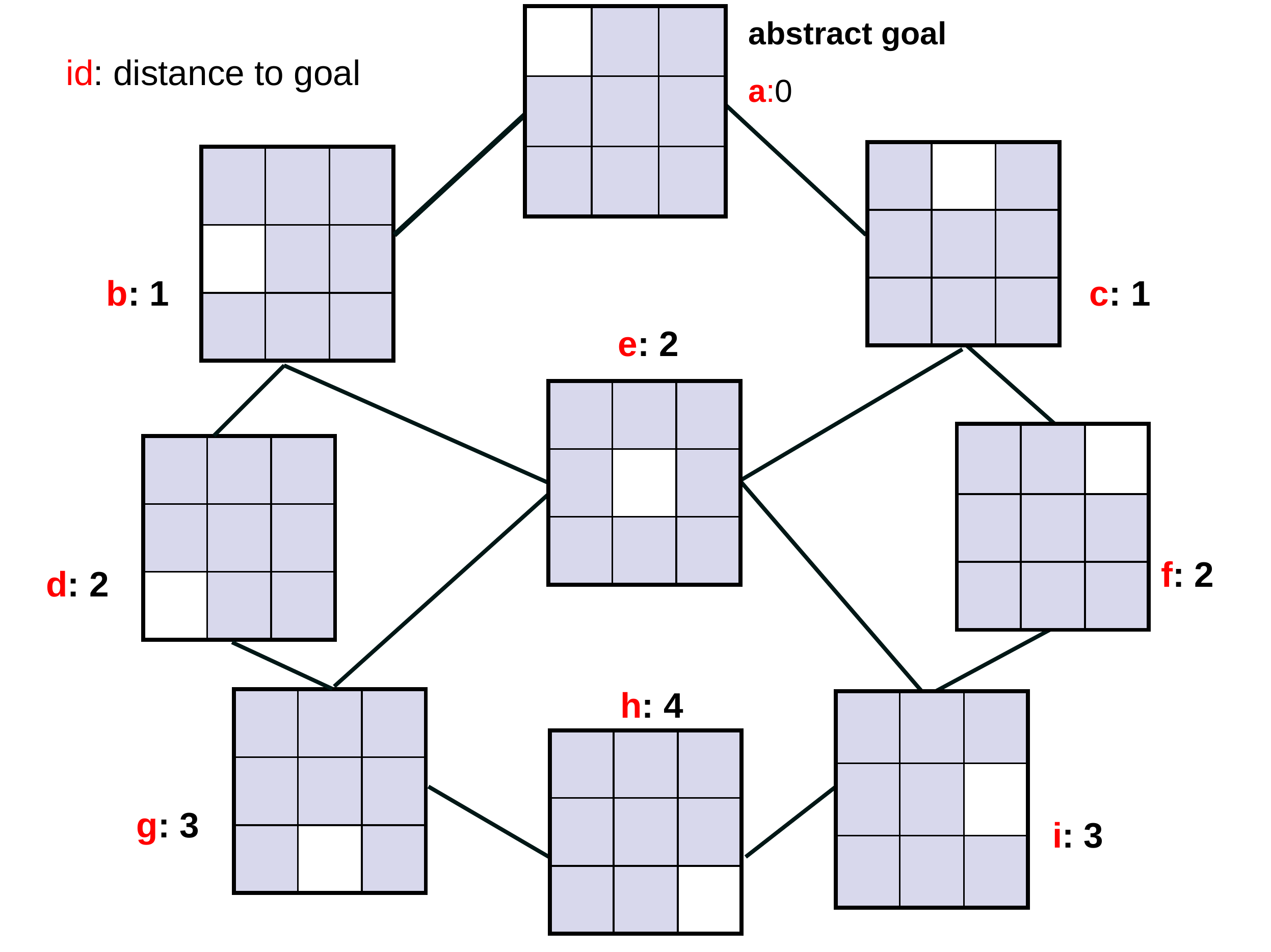}
}
\caption{An abstraction of the 8-puzzle.  The white square in each state
is the blank and the non-white squares are the tiles, which are all
indistinguishable from each other in this abstraction.}
\label{fig-8puzzle}
\end{figure}

Given several abstractions of a state space, the heuristic $\hmaxtg$ can
be defined as the maximum of the abstract distances for $t$ given by the
abstractions individually. This is the standard method for defining a
heuristic function given multiple abstractions \cite{maxingAIJ06}. For
example, consider state $A$ of the $3\times 3$ sliding tile puzzle shown
in the top left of Figure \ref{fig-maxPDB} and the goal state shown below
it. The middle column shows an abstraction of these two states ($A_1$ and
$g_1$) in which tiles 1, 3, 5, and 7, and the blank, are distinct while
the other tiles are indistinguishable from each other. We refer to the
distinct tiles as ``distinguished tiles" and the indistinguishable tiles
as ``don't care''  tiles. The right column shows the complementary
abstraction, in which tiles 1, 3, 5, and 7 are the ``don't cares" and
tiles 2, 4, 6, and 8 are distinguished. The arrows in the figure trace out
a least-cost path to reach the abstract goal $g_i$ from state $A_i$ in
each abstraction. The cost of solving $A_1$ is 16 and the cost of solving
$A_2$ is 12. Therefore, $\hmaxa{A}{\goal}$ is 16, the maximum of these two
abstract distances.

\begin{figure}[!ht]
\centerline{
\includegraphics[width=9cm]{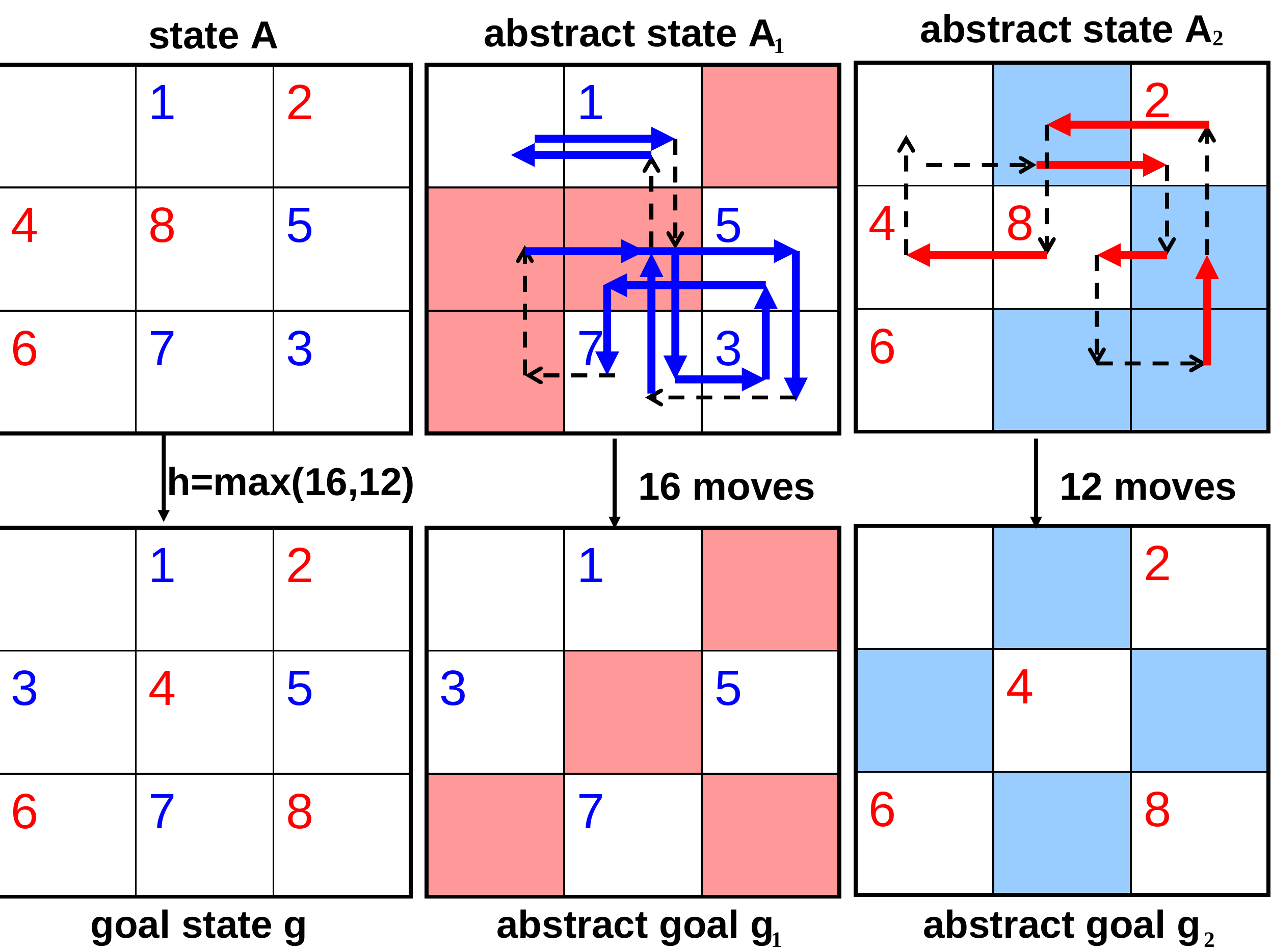}
}
\caption{Computation of $\hmaxa{A}{\goal}$, the standard, maximum-based
heuristic value for state $A$ (top left) using the two abstractions shown in
the middle and right columns. Solid arrows denote distinguished moves,
dashed arrows denote ``don't care" moves.}
\label{fig-maxPDB}
\end{figure}

\subsection{Additive Abstractions}
\label{additivePDB}

Figure \ref{fig-infeasible} illustrates how additive abstractions
can be defined for the sliding tile puzzle
\cite{disjointPDB,ADDPDB,KorfTaylor1996}. State $A$ and the abstractions are the
same as in Figure \ref{fig-maxPDB}, but the costs of the operators
in the abstract spaces are defined differently. Instead of all
abstract operators having a cost of 1, as was the case previously,
an operator only has a cost of 1 if it moves a distinguished tile;
such moves are called ``distinguished moves" and are shown as solid
arrows in Figures \ref{fig-maxPDB} and \ref{fig-infeasible}.
An operator that moves a ``don't care'' tile (a ``don't care" move) has a
cost of 0 and is shown as a dashed arrow in the figures.
Least-cost paths in abstract spaces defined this way
therefore minimize the number of distinguished moves without
considering  how many ``don't care'' moves are made. For example, the least-cost
path for $A_1$ in Figure \ref{fig-infeasible}
contains fewer distinguished moves (9 compared to 10)
than the least-cost path for $A_1$ in Figure \ref{fig-maxPDB}---and is
therefore lower cost according to the cost function
just described---but contains more moves in total
(18 compared to 16) because it has more ``don't care'' moves (9 compared to 6).
As Figure \ref{fig-infeasible}
shows, 9 distinguished moves are needed to solve $A_1$ and 5
distinguished moves are needed to solve $A_2$. Because no tile is
distinguished in both abstractions, a move that has a cost of 1 in
one space has a cost of 0 in the other space, and it is therefore
admissible to add the two distances. The heuristic calculated using
additive abstractions is referred to as $\hadd$; in this example,
$\hadda{A}{\goal}=9+5=14$. Note that $\hadda{A}{\goal}$ is less than
$\hmaxa{A}{\goal}$ in this example, showing that heuristics based on
additive abstractions are not always superior to the standard,
maximum-based method of combining multiple abstractions even though
in general they have proven very effective on the sliding tile
puzzles \cite{disjointPDB,ADDPDB,KorfTaylor1996}.

\begin{figure}[!ht]
\centerline{
\includegraphics[width=9cm]{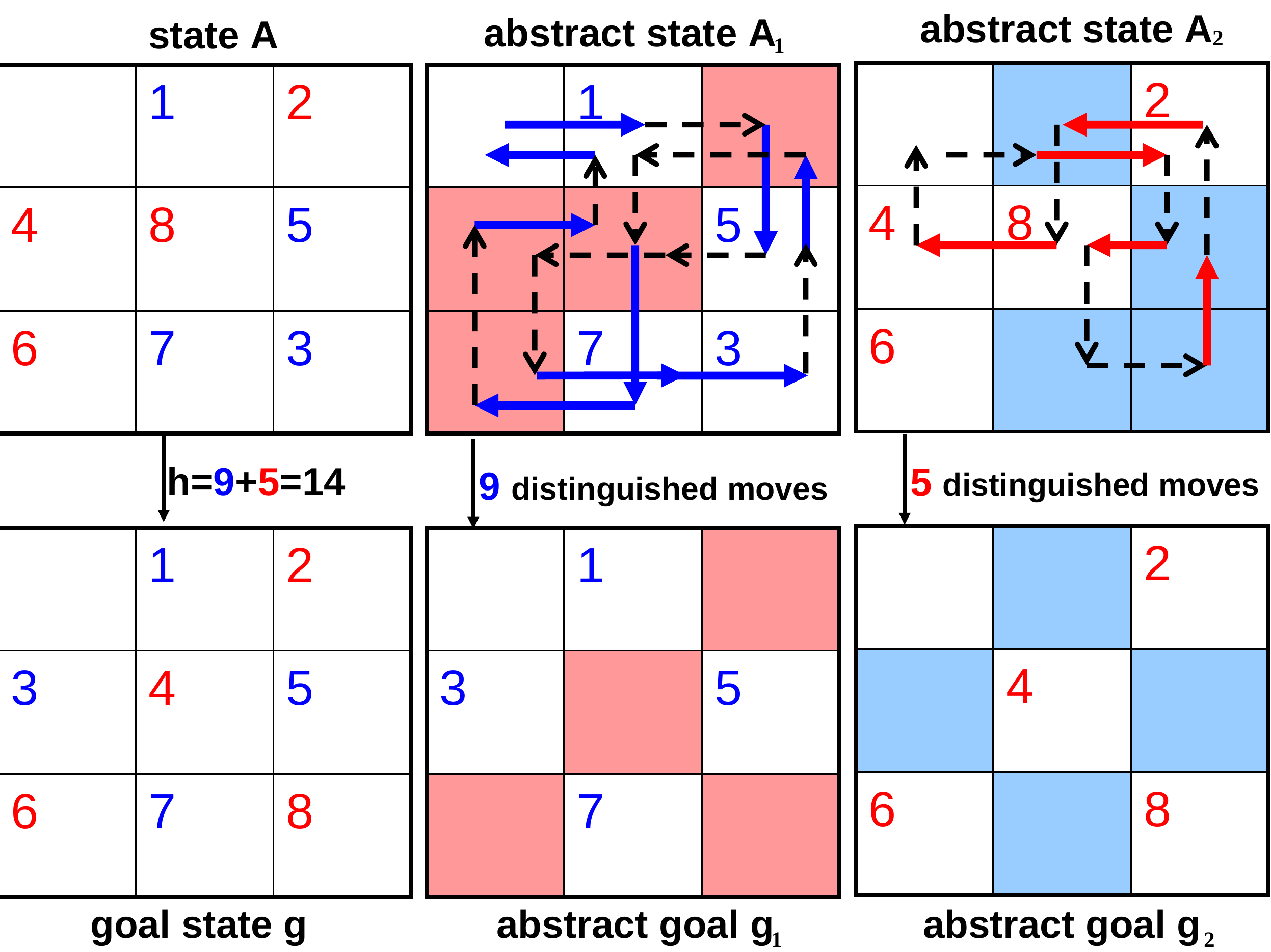}
}
\caption{Computation of $\hadda{A}{\goal}$, the additive
heuristic value for state $A$.
Solid arrows denote distinguished moves,
dashed arrows denote ``don't care" moves.}
\label{fig-infeasible}
\end{figure}

The general method defined by Korf, Felner, and colleagues
\cite{disjointPDB,ADDPDB,KorfTaylor1996}
creates a set of $k$ additive abstractions by partitioning the tiles into
$k$ disjoint groups and defining one abstraction for each group by making
the tiles in that group distinguished in the abstraction. An important
limitation of this and most other existing methods of defining additive
abstractions is that they do not apply to spaces in which an operator can
move more than one tile at a time, unless there is a way to guarantee that
all the tiles that are moved by the operator are in the same group.

An example of a state space that has no additive abstractions
according to previous definitions is the Pancake puzzle. In the
$N$-Pancake puzzle, a state is a permutation of $N$ tiles
($0,1,...,N-1$) and has $N-1$ successors, with the $l^{th}$ successor
formed by reversing the order of the first $l+1$ positions of the
permutation ($ 1\leq l \leq N-1$). For example, in the 4-Pancake
puzzle shown in Figure \ref{fig-state4successor}, the state at the
top of the figure has three successors, which are formed by
reversing the order of the first two tiles, the first three tiles,
and all four tiles, respectively.
Because the operators move more than one tile and any tile
can appear in any location
there is no
non-trivial way to partition the tiles so that all the tiles moved
by an operator are distinguished in just one abstraction. Other
common state spaces that have no additive abstractions according to
previous definitions---for similar reasons---are Rubik's Cube and
TopSpin.

\begin{figure}[!ht]
\centerline{
\includegraphics[width=6cm]{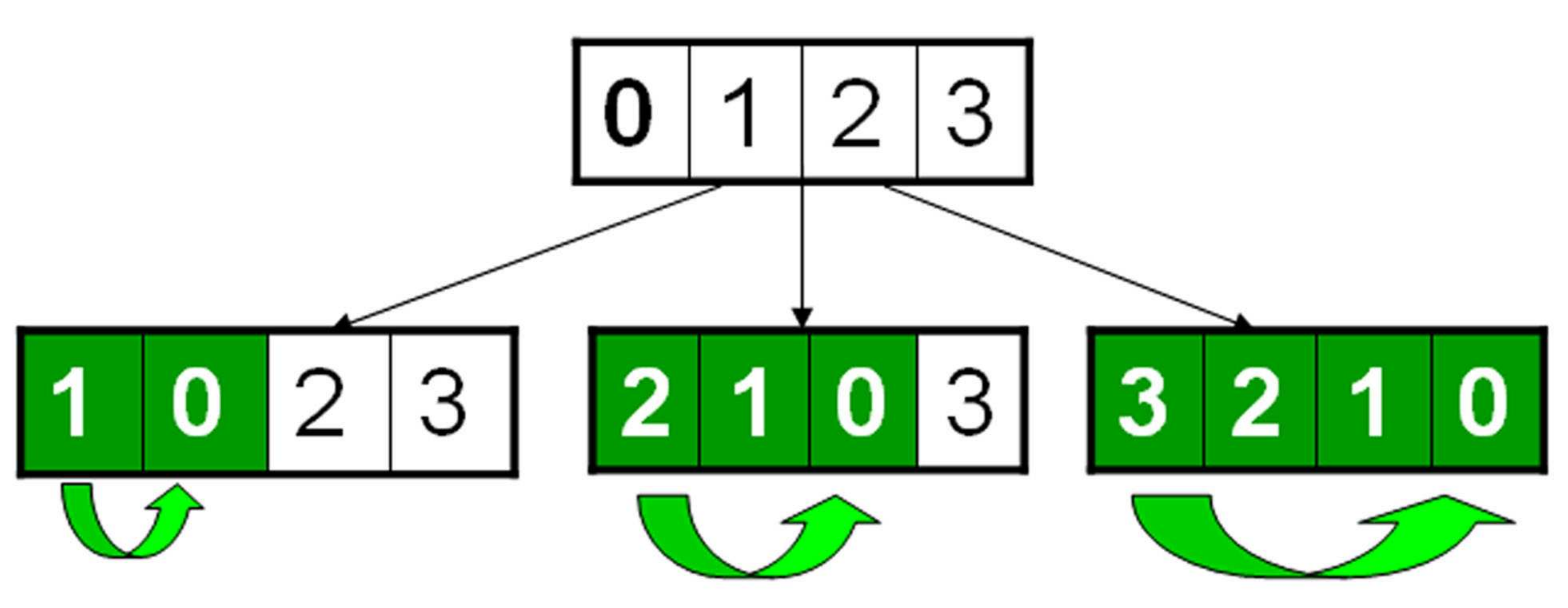}
}
\caption{In the 4-Pancake puzzle each state has three successors.}
\label{fig-state4successor}
\end{figure}

The general definition of additive abstractions presented in the
next section overcomes the limitations of previous definitions.
Intuitively, abstractions will be additive
provided that the cost of each operator is divided among the
abstract spaces.
Our definition provides a formal basis for this intuition.
There are numerous ways to do this even when
operators move many tiles (or, in other words, make changes to many state variables).
For example, the operator cost might be divided
proportionally across the abstractions based on the percentage of
the tiles moved by the operator that are distinguished in each
abstraction. We call this method of defining abstract costs
``cost-splitting". For example, consider two abstractions of the
4-Pancake puzzle, one in which tiles 0 and 1 are distinguished, the
other in which tiles 2 and 3 are distinguished. Then the middle
operator in Figure \ref{fig-state4successor} would have a cost of
$\frac{2}{3}$ in the first abstract space and $\frac{1}{3}$ in the
second abstract space, because of the three tiles this operator
moves, two are distinguished in the first abstraction and one is
distinguished in the second abstraction.

A different method for dividing operator costs among abstractions
focuses on a specific location (or locations)
in the puzzle and assigns the full cost of the operator
to the abstraction in which the tile that moves into this location
is distinguished.
We call this a ``location-based" cost definition.
In the Pancake puzzle it is natural to use the leftmost location as the
special location since every operator changes the tile in this
location.
The middle operator in Figure \ref{fig-state4successor} would
have a cost of $0$ in the abstract space in which tiles 0 and 1
are distinguished and a cost of $1$
in the abstract space in which tiles 2 and 3
are distinguished because the operator moves tile 2 into
the leftmost location.

Both these methods apply to Rubik's Cube and TopSpin, and many
other state spaces in addition to the Pancake puzzle, but the
$\hadd$ heuristics they produce are not always superior to the
$\hmax$ heuristics based on the same tile partitions. The
theory and experiments in the remainder of the paper shed some light
on the general question of when $\hadd$ is preferable to
$\hmax$.

\section{Formal Theory of Additive Abstractions}
\label{General_Definitions} In this section, we give formal
definitions and lemmas related to state spaces, abstractions, and
the heuristics defined by them, and discuss their meanings and
relation to previous work.
The definitions of state space etc.\ in
Section \ref{DEFstatespace} are standard, and the definition of
state space abstraction in Section \ref{DEFabstraction} differs from
previous definitions only in one important detail:
each state transition in an abstract space has two costs associated with it
instead of just one.
The main new contribution is the
definition of additive abstractions in Section \ref{addabstraction}.

The underlying structure of our abstraction definition is
a directed graph (digraph) homomorphism.
For easy reference, we quote here standard definitions of digraph and digraph homomorphism \cite{PavolHell}.

\begin{definition}
 A {\em digraph} $G$ is a finite set $V = V (G)$ of
vertices, together with a binary relation $E = E(G)$ on $V.$ The elements $(u, v)$ of E are called the {\em arcs} of G.
\end{definition}

\begin{definition}
 Let $G$ and $H$ be any digraphs. A {\em homomorphism} of $G$ to $H$, written as
f : $G \rightarrow H$ is a mapping $f $: $V (G) \rightarrow V (H)$ such that $(f(u),f(v))\in E(H)$ whenever $(u,v) \in E(G)$.
\label{def-digraph-homo}
\end{definition}

Note that the digraphs $G$ and $H$ may have self-loops, $(u,u)$, and
a homomorphism is not required to be surjective in either vertices or arcs.
We typically refer to arcs as edges, but it should be kept in mind that,
in general, they are directed edges, or ordered pairs.

\subsection{State Space}

\label{DEFstatespace}
\begin{definition}
\label{DEF.statespace}
A {\em state space} is a weighted directed graph $\gspace =
\tuple{\gstates, \gpairs, \gcost }$ where $\gstates$ is a
finite set of states, $\gpairs \subseteq \gstates \times \gstates$
is a set of directed edges (ordered pairs of states) representing
state transitions, and
$\gcost : \gpairs \longrightarrow {\nat} = \{ 0,1,2,3,\dots\}$
is the edge cost function. 
\end{definition}
In typical practice, $\gspace$ is defined implicitly.
Usually each distinct state in $\gstates$ corresponds to an assignment of values to a set
of state variables.
$\gpairs$ and  $\gcost$ derive from a successor function, or a set of planing operators.
In some cases, $\gstates$ is restricted to the set of states reachable from a given state.
For example, in the 8-puzzle, the set of edges $\gpairs$ is
defined by the rule
``a tile that is adjacent to the empty location can be moved into the
empty location", and the set of states $\gstates$ is defined in one
of two ways: either as
the set of states reachable from the goal state,
or as
the set of permutations of the tiles and the blank,
in which case $\gstates$ consists of two components that are not connected
to one another.
The standard cost function $\gcost$ for the 8-puzzle assigns a cost of $1$
to all edges, but it is easy to imagine cost functions for the 8-puzzle
that depend on the tile being moved or the locations involved in the
move.

A {\em path} from state $\gst$ to state $\goal$
is a sequence of edges beginning at $\gst$ and ending at $\goal$.
Formally,
$\gpath$ is a path from state $\gst$ to state $\goal$ if
$\gpath =
\tuple{ \gprj{1}, \dots, \gprj{n} }, \gprj{j} \in \gpairs$ where
$\gprj{j} = (\gstj{j-1}, \gstj{j}), j \in \{ 1, \dots, n \}$ and $\gstj{0}
= \gst, \gstj{n} = \goal$. Note the use of superscripts rather
than subscripts to distinguish states and edges within a state
space. The {\em length} of $\gpath$ is the number of edges $n$ and
its {\em cost} is $\gcost(\gpath) = \sum_{j=1}^n \gcost(\gprj{j})$.
We use $ \gpaths{\gst}{\goal}$ to denote the set of all paths from
$\gst$ to  $\goal$ in $\gspace$.
\begin{definition}
\label{DEF.OPT}
The {\em optimal (minimum) cost} of a path from state $\gst$ to state $\goal$ in $\gspace$ is defined by
\begin{eqnarray*}
\gopt(\gst,\goal) = \min_{ \gpath \in \gpaths{\gst}{\goal}} \gcost(\gpath)
\end{eqnarray*}
\end{definition}

A {\em pathfinding problem} is a triple $\tuple{\gspace, \gst,
\goal}$, where
$\gspace$ is a state space and
$\gst,\goal\in\gstates$, with the objective of
finding the minimum cost of a path from $\gst$ to $\goal$,
 or in some cases
finding a minimum cost path
$\gpath \in \gpaths{\gst}{\goal}$
such that $\gcost(\gpath) = \gopt(\gst,\goal)$.
Having just one goal state may
seem restrictive,
but problems having a set of goal states can be
accommodated with this definition
by adding a virtual goal state to the state space
with zero-cost edges from the actual goal states to the virtual
goal state.

\subsection{State Space Abstraction}
\label{DEFabstraction}

\begin{definition}
\label{DEF.abstractsystem}
An {\em Abstraction System} is a pair $\absys$ where
$\gspace = \tuple{\gstates, \gpairs, \gcost }$ is a state space
and
$\absset = \{ \tuple{\abspace{i},\abmap{i}} \setdelim \abmap{i}:\gspace \rightarrow \abspace{i}, 1\le i \le k\}$ is a set of {\em abstractions}, where each abstraction is a pair consisting of an {\em abstract state space} and an {\em abstraction mapping}, where ``abstract state space" and ``abstraction mapping"
are defined below.
\end{definition}

Note that these abstractions are not intended to form a hierarchy and should be considered a set of independent abstractions.

\begin{definition}
\label{DEF.absstatespace}
An {\em abstract state space} is a directed graph with two weights
per edge, defined by a four-tuple $\abspace{i} =
\tuple{\abstates{i}, \abpairs{i}, \abcost{i}, \abresid{i}}$.
\end{definition}
$\abstates{i}$ is the set of abstract states and $\abpairs{i}$ is
the set of abstract edges, as in the definition of a state space. In
an abstract space there are two costs associated with each
$\abpr{i}\in\abpairs{i}$, the {\em primary cost} $\abcost{i} :
\abpairs{i} \longrightarrow \nat$ and the {\em  residual cost}
$\abresid{i} : \abpairs{i} \longrightarrow \nat$. The idea of having
two costs per abstract edge, instead of just one,
is inspired by the
practice, illustrated in Figure \ref{fig-infeasible}, of having two
types of edges in the abstract space and counting distinguished
moves differently than ``don't care'' moves. In that example,
our primary cost is the cost associated with the distinguished moves,
and our residual cost
is the cost associated with the ``don't care'' moves.
The usefulness of considering the cost of
``don't care'' moves arises when the abstraction system is additive, as suggested by Lemmas \ref{LEM.ADDBETTER} and
\ref{LEM.INFEASIBLE} below.
These indicate when the additive heuristic is infeasible and can be improved,   the effectiveness of which  will become apparent in the
experiments reported in Section \ref{infeasible}.

Like edges, each
abstract path $\abpi{i} = \tuple{ \abprij{i}{1}, \dots,
\abprij{i}{n} }$ in $\abspace{i}$ has a primary and residual cost:
$\abcost{i}(\abpi{i}) = \sum_{j=1}^n \abcost{i}(\gprj{j}_i)$, and
$\abresid{i}(\abpi{i}) = \sum_{j=1}^n \abresid{i}(\gprj{j}_i)$.

\begin{definition}
\label{DEF.absmapping}
An {\em abstraction mapping}
$\abmap{i} : \gspace \longrightarrow \abspace{i}$
between state space $\gspace$ and abstract state space $\abspace{i}$
is defined by a mapping between the states of $\gspace$
and the states of $\abspace{i}$, $\abmap{i}: T \rightarrow T_i$,
that satisfies the two following conditions.
\end{definition}
The first condition is that the mapping is a homomorphism and thus connectivity
in the original space is preserved, i.e.,

\begin{equation}\label{CONNcondition}
\forall  (u,v) \in \gpairs, (\abmap{i}(u), \abmap{i}(v)) \in \abpairs{i}
\end{equation}

In other words, for each edge in the original space $\gspace$ there
is a corresponding edge in the abstract space $\abspace{i}$. Note
that if $u \ne v$ and $\abmap{i}(u) = \abmap{i}(v)$ then a
non-identity edge in $\gspace$ gets mapped to an identity edge
(self-loop) in $\abspace{i}$. We use the shorthand notation
$\abstij{i}{j} = \abmap{i}(\gstj{j})$ for the abstract state in
$\abstates{i}$ corresponding to $\gstj{j} \in \gstates$, and
$\abprij{i}{j} = \abmap{i}(\gprj{j}) = (\abmap{i}(u^j),
\abmap{i}(v^j))$ for the abstract edge in $\abpairs{i}$
corresponding to $\gprj{j} = (u^j,v^j) \in \gpairs$.

The second condition that the state mapping must satisfy
is that abstract edges must not cost more than any of the edges they
correspond to in the original state space,  i.e.,

\begin{equation}\label{ADMcondition}
 \forall \gpr \in \gpairs,
\abcost{i}(\abpr{i}) + \abresid{i}(\abpr{i})
\le
\gcost(\gpr)
\end{equation}

As a consequence, if multiple edges in the original space map to
the same abstract edge $\gabpr \in \abpairs{i}$, as is usually the case,
$\abcost{i}(\gabpr) + \abresid{i}(\gabpr)$ must be less than
or equal to all of them,
i.e.,

\begin{equation*}
\forall \gabpr \in \abpairs{i}, \  \abcost{i}(\gabpr) + \abresid{i}(\gabpr)
\le
\min_{\gpr \in \gpairs, \abmap{i}(\gpr) = \gabpr} \gcost(\gpr)
\end{equation*}
Note that if no edge maps to an edge in the abstract space, then no bound on the cost of that edge is imposed.

For example, the state mapping used to define the abstraction in the
middle column of Figure \ref{fig-infeasible} maps an 8-puzzle state
to an abstract state by renaming tiles 2, 4, 6, and 8 to ``don't
care". This mapping satisfies condition (1) because ``don't care"
tiles can be exchanged with the blank whenever regular tiles can.
It satisfies condition (2) because each move is either
a distinguished move
($\abcost{i}(\abpr{i}) = 1$ and $\abresid{i}(\abpr{i}) = 0$)
or a ``don't care'' move
($\abcost{i}(\abpr{i}) = 0$ and $\abresid{i}(\abpr{i}) = 1$)
and in both cases
$\abcost{i}(\abpr{i}) +
\abresid{i}(\abpr{i}) = 1$, the cost of the edge $\pi$ in the original space.

The set of abstract states $\abstates{i}$ is usually equal to
$\abmap{i}(\gstates)  = \{ \abmap{i}(t) \setdelim t \in \gstates
\}$, but it can be a superset, in which case the
abstraction is said to be {\em non-surjective} \cite{SARA2000}.
Likewise, the set of abstract edges $\abpairs{i}$ is usually equal
to $\abmap{i}(\gpairs) = \{ \abmap{i}(\pi) \setdelim \pi \in \gpairs
\}$ but it can be a superset even if $\abstates{i} =
\abmap{i}(\gstates)$.
In some cases, one deliberately chooses an abstract space that has
states or edges that have no counterpart in the original space.
For example,
the methods that define abstractions by dropping operator preconditions
must, by their very design,
create abstract spaces that
have edges that do not correspond to any edge in the original space
({\em e.g.}\ \citeR{PearlHeuristics}).
In other cases, non-surjectivity is an inadvertent consequence
of the abstract space being defined implicitly as the set of states
reachable from the abstract goal state by applying operator inverses.
For example, if a tile in the $2 \times 2$ sliding tile puzzle
is mapped to the blank
in the abstract space, the puzzle now has two blanks and states are
reachable in the abstract space that have no counterpart in the
original space \cite{SARA2000}. For additional
examples and an extensive discussion of non-surjectivity
see the previous paper by \citeA{robistvan04}.

All the lemmas and definitions that
follow assume 
an abstraction system
$\absys$ containing $k$ abstractions has been given.
Conditions \eqref{CONNcondition} and \eqref{ADMcondition} guarantee
the following.
\begin{lemma}
\label{LEM.PathCost}
For any path
$\gpath \in \gpaths{u^1}{u^2}$ in $\gspace$,
there is a corresponding abstract path $\abpath{i}$
from $u^1_i$ to $u^2_i$ in $\abspace{i}$
and $\abcost{i}(\abpath{i}) + \abresid{i}(\abpath{i}) \le
\gcost(\gpath)$.
\end{lemma}
\begin{proof}
By definition, $\gpath \in \gpaths{u^1}{u^2}$ in $\gspace$
is a sequence of edges
$\tuple{ \gprj{1}, \dots, \gprj{n} }, \gprj{j} \in \gpairs$
where $\gprj{j} = (\gstj{j-1}, \gstj{j}),  j \in \{ 1, \dots, n \}$
and $\gstj{0} = u^1, \gstj{n} = u^2$.
Because   $\abpairs{i} \supseteq \abmap{i}(\gpairs)$,
each of the corresponding abstract edges exists ($\gprj{j}_i \in \abpairs{i}$).
Because $\abprij{i}{1} = (u^1_i,\abstij{i}{1})$ and
$\abprij{i}{n} = (\abstij{i}{n-1},u^2_i)$,
the sequence
$\abpath{i} =
\tuple{ \abprij{i}{1}, \dots, \abprij{i}{n} }$
is a path from $u^1_i$ to $u^2_i$.

By definition, $\gcost(\gpath) = \sum_{j=1}^n \gcost(\gprj{j})$.
For each $\gprj{j}$, Condition (2) ensures that
$\gcost(\gprj{j}) \ge \abcost{i}(\abprij{i}{j}) + \abresid{i}(\abprij{i}{j})$,
and therefore
$\gcost(\gpath) \ge
\sum_{j=1}^n (\abcost{i}(\abprij{i}{j}) + \abresid{i}(\abprij{i}{j}))
= \sum_{j=1}^n \abcost{i}(\abprij{i}{j}) + \sum_{j=1}^n \abresid{i}(\abprij{i}{j})
= \abcost{i}(\abpath{i}) + \abresid{i}(\abpath{i})$.
\end{proof}

For example, consider state $A$ and goal $g$ in Figure \ref{fig-infeasible}.
Because of condition \eqref{CONNcondition},
any path from state $A$ to $g$ in the original space is also
a path from abstract state $A_1$ to abstract goal state $g_1$
and from abstract state $A_2$ to $g_2$
in the abstract spaces.
Because of condition \eqref{ADMcondition},
the cost of the path in the original space is
greater than or equal to
the sum of the primary cost and the residual cost of the
corresponding abstract path in each abstract space.

We use $\abpathset{i}{u}{v}$ to mean the set of all paths from $u$ to $v$ in space $\abspace{i}$.
\begin{definition}
\label{DEF.ABSOPT}
The  {\em optimal abstract cost} from
abstract state $\gabst$ to abstract state $\gabsg$ in $\abspace{i}$
is defined as
\begin{equation*}
\abopt{i}(\gabst,\gabsg) =
\min_{ \gabpath \in \abpathset{i}{u}{v} }
\abcost{i}(\gabpath) + \abresid{i}(\gabpath)
\end{equation*}
\end{definition}

\begin{definition}
\label{DEF.heuristic}
We define the heuristic
obtained from  abstract space $\abspace{i}$ for the cost from state $\gst$ to $\goal$ as
\begin{eqnarray*}
\hitg=\abopt{i}(\abst{i},\abgoal{i}).
\end{eqnarray*}
\end{definition}
Note that in these definitions,
the path minimizing the cost is not required to be the image, $\abpath{i}$,
of a path $\gpath$ in $\gspace$.

The following prove that the heuristic generated by each individual
abstraction is
admissible (Lemma \ref{LEM.CversusCR})
and consistent (Lemma \ref{LEM.individualConsistent}).

\begin{lemma}
\label{LEM.CversusCR}
$\hitg \le
\gopt(\gst,\goal)$
for all $\gst, \goal  \in \gstates$ and all $ i  \in \{ 1, \dots, k \}$.
\end{lemma}
\begin{proof}
By Lemma \ref{LEM.PathCost},
$\gcost(\gpath) \ge
\abcost{i}(\abpath{i}) +\abresid{i}(\abpath{i})$,
and therefore

\[
\min_{ \gpath \in \gpaths{\gst}{\goal}}
\gcost(\gpath)
\ge
\min_{\gpath \in \gpaths{\gst}{\goal} }
\abcost{i}(\abpath{i})+\abresid{i}(\abpath{i}).
\]

\noindent
The left hand side of this inequality is $ \gopt(\gst,\goal)$
by definition, and the right hand side is proved in the following Claim \ref{LEM.CversusCR}.1
to be greater than or equal to $\hitg$.
Therefore, $\gopt(\gst,\goal) \ge  \hitg$.

{\bf Claim \ref{LEM.CversusCR}.1} {\em
$\min_{\gpath \in \gpaths{\gst}{\goal}}
\abcost{i}(\abpath{i})+\abresid{i}(\abpath{i})
\ge \hitg$}
for all $\gst, \goal  \in \gstates$.

{\bf Proof of Claim \ref{LEM.CversusCR}.1: }
By Lemma~\ref{LEM.PathCost} for every path $\gpath$ there is a corresponding abstract path. There may also be additional paths in the abstract space, that is,
$\{ \abpath{i} \setdelim \gpath \in \gpaths{\gst}{\goal} \} \subseteq
\abpathset{i}{\abst{i}}{\abgoal{i}}$.
It follows that
$\{ \abcost{i}(\abpath{i})+\abresid{i}(\abpath{i})
\setdelim \gpath  \in \gpaths{\gst}{\goal} \}
\subseteq
\{ \abcost{i}(\gabpath) + \abresid{i}(\gabpath)
\setdelim \gabpath \in \abpathset{i}{\abst{i}}{\abgoal{i}}\}$.
Therefore,
\[
\min_{  \gpath \in \gpaths{\gst}{\goal} }
 \abcost{i}(\abpath{i})+\abresid{i}(\abpath{i})
  \ge
\min_{ \gabpath \in \abpathset{i}{\abst{i}}{\abgoal{i}} }
\abcost{i}(\gabpath) + \abresid{i}(\gabpath)
 =   \abopt{i}(\abst{i},\abgoal{i}) = \hitg
 \]

\end{proof}

\begin{lemma}
\label{LEM.individualConsistent}
$\hi{t^1}{\goal} \le \gopt(t^1, t^2) + \hi{t^2}{\goal}$
for all $t^1, t^2,\goal  \in \gstates$ and all $i  \in \{ 1, \dots, k \}$.
\end{lemma}
\begin{proof}
By the definition of $\abopt{i}$ as a minimization and the definition of $\hitg$, it follows that
$\hi{t^1}{\goal} = \abopt{i}(t^1_i,\abgoal{i}) \le \abopt{i}(t^1_i, t^2_i) +  \abopt{i}(t^2_i, g_i)
=  \abopt{i}(t^1_i, t^2_i) +\hi{t^2}{g}$.

To complete the proof, we observe that
by Lemma \ref{LEM.CversusCR}, $\gopt(t^1, t^2) \ge \hi{t^1}{t^2} = \abopt{i}(t^1_i, t^2_i)$.
\end{proof}

\begin{definition}
\label{DEF.hmax}
The $\hmax$ heuristic
from state $t$ to state $\goal$ defined by an abstraction system  $\absys$
is
\begin{eqnarray*}
\hmaxtg =\max_{i=1}^k \hitg
\end{eqnarray*}
\end{definition}
From Lemmas \ref{LEM.CversusCR} and \ref{LEM.individualConsistent}
it immediately follows that $\hmax$ is admissible and consistent.

\subsection{Additive Abstractions}
\label{addabstraction}

In this section, we formalize the
notion of ``additive abstraction"
that was introduced intuitively in Section \ref{additivePDB}.
The example there showed that  $\haddtg$, the sum of the
heuristics for state $t$ defined by multiple abstractions,
was admissible provided the cost functions in the abstract spaces
only counted the ``distinguished moves".
In our formal framework, the ``cost of distinguished moves" is captured
by the notion of primary cost.
\begin{definition}
For any pair of states $\gst, \goal \in \gstates$
the {\em additive heuristic}  given an abstraction system
is defined to be

\begin{equation*}\label{haddDEF}
\haddtg=\sum_{i=1}^k \abminc{i}(\abst{i},\abgoal{i}).
\end{equation*}

\noindent
where
\[ \abminc{i}(\abst{i},\abgoal{i}) =
\min_{\gabpath \in \abpathset{i}{\abst{i}}{\abgoal{i}} }  \abcost{i}(\gabpath)
\]
is the minimum primary cost of a path in the abstract space from $\abst{i}$ to $\abgoal{i}$.
\end{definition}

In Figure \ref{fig-infeasible},
for example,
$C^*_1(A_1,g_1)=9$ and $C^*_2(A_2,g_2)=5$ because the minimum
number of distinguished moves to reach $g_1$ from $A_1$ is $9$
and the minimum number of distinguished moves to reach $g_2$ from $A_2$ is $5$.


Intuitively, $\hadd$ will be admissible if the cost of edge
$\gpr$ in the original space is divided among the abstract edges
$\abpr{i}$ that correspond to $\gpr$, as is done by the ``cost-splitting"
and ``location-based" methods for defining abstract costs that were
introduced at the end of Section \ref{additivePDB}.
This leads to the following formal definition.
\begin{definition}
\label{DEF.ADDITIVEABS}
An
abstraction system $\absys$ is {\em additive} if $ \forall \gpr
\in \gpairs, \sum_{i=1}^k \abcost{i}(\abpr{i}) \le \gcost(\gpr) $.
\end{definition}

The following prove that $\hadd$ is admissible (Lemma
\ref{LEM.ADDLB}) and consistent (Lemma \ref{LEM.consistent}) when
the abstraction system $\absys$ is {\em additive}.

\begin{lemma}
\label{LEM.ADDLB}
If $\absys$ is additive then
$\haddtg \le \gopt(\gst,\goal)$
for all $ t ,\goal \in \gstates$.
\end{lemma}
\begin{proof}
Assume that $ \gopt(\gst,\goal) = \gcost(\gpath) $,
where $\gpath = \tuple{ \gprj{1}, \dots, \gprj{n}}  \in \gpaths{\gst}{\goal} $.
Therefore, $\gopt(\gst,\goal)= \sum_{j=1}^n \gcost(\gprj{j})$. Since  $\absys$ is additive,  it follows by definition that
\begin{eqnarray*}
\sum_{j=1}^n  \gcost(\gprj{j}) & \ge & \sum_{j=1}^n \sum_{i=1}^k \abcost{i}(\abpr{i}^j)
=  \sum_{i=1}^k \sum_{j=1}^n  \abcost{i}(\abpr{i}^j)\\
 & \ge & \sum_{i=1}^k \abminc{i}(\abst{i},\abgoal{i}) = \haddtg
\end{eqnarray*}
where the last line follows from the definitions of  $\abminc{i}$ and $\hadd$.
\end{proof}

\begin{lemma}
\label{LEM.consistent}
If $\absys$ is additive then
$\hadda{t^1}{\goal} \le \gopt(t^1, t^2) +  \hadda{t^2}{\goal}$
for all $t^1, t^2, \goal  \in \gstates$.
\end{lemma}
\begin{proof}
$\abminc{i}(t^1_i,\abgoal{i})$
obeys the triangle inequality:
$\abminc{i}(t^1_i, \abgoal{i}) \le
\abminc{i}(t^1_i,t^2_i) +  \abminc{i}(t^2_i, \abgoal{i})$
for all $t^1, t^2,\goal  \in \gstates$.
It follows that
$\sum_{i=1}^k \abminc{i}(t^1_i, \abgoal{i}) \le
\sum_{i=1}^k \abminc{i}(t^1_i,t^2_i) +  \sum_{i=1}^k \abminc{i}(t^2_i, \abgoal{i})$.

Because $\sum_{i=1}^k \abminc{i}(t^1_i, \abgoal{i})= \hadda{t^1}{\goal}$ and $\sum_{i=1}^k \abminc{i}(t^2_i, \abgoal{i})=\hadda{t^2}{\goal}$, it follows that $\hadda{t^1}{\goal} \le \sum_{i=1}^k \abminc{i}(t^1_i,t^2_i) +  \hadda{t^2}{\goal}$.

Since  $\absys$ is additive,  by Lemma \ref{LEM.ADDLB},  
$\gopt(t^1, t^2) \ge
\sum_{i=1}^k \abminc{i}(t^1_i,t^2_i)$.

Hence
 $\hadda{t^1}{\goal} \le \gopt(t^1, t^2) +  \hadda{t^2}{\goal}$
for all $t^1, t^2, \goal \in \gstates$.
\end{proof}

We now develop a simple test that has important consequences for
additive heuristics.
Define $\primeset{i}{\abst{i}}{\abgoal{i}} =
\{ \gabpath  \setdelim\gabpath \in \abpathset{i}{\abst{i}}{\abgoal{i}} $ and
$\abcost{i}(\gabpath) = \abminc{i}(\abst{i},\abgoal{i})
\}$, the set of abstract paths from $\abst{i}$ to $\abgoal{i}$ whose
primary cost is minimal.
\begin{definition}
\label{DEF.CONDOPTRESCOST}
The {\em conditional optimal
residual cost} is the minimum residual cost among the paths in
$\primeset{i}{\abst{i}}{\abgoal{i}}$:
\[
\abminr{i}(\abst{i},\abgoal{i}) =
\min_{ \gabpath
\in \primeset{i}{\abst{i}}{\abgoal{i}}}
\abresid{i}(\gabpath)
\]
\end{definition}
Note that the value of
($\abminc{i}(\abst{i},\abgoal{i}) +
\abminr{i}(\abst{i},\abgoal{i})$) is sometimes, but not always,
equal to the optimal abstract cost $\abopt{i}(\abst{i},\abgoal{i})$.
In Figure \ref{fig-infeasible}, for example, $\abopt{1}(A_1,g_1)=16$
(a path with this cost is shown in Figure \ref{fig-maxPDB}) and
$\abminc{1}(A_1,g_1) + \abminr{1}(A_1,g_1)=18$, while
$\abminc{2}(A_2,g_2) + \abminr{2}(A_2, g_2)=\abopt{2}(A_2,g_2)=12$.
\label{infeasibleDef} As the following lemmas show, it is possible
to draw important conclusions about $\hadd$ by comparing its value
to ($\abminc{i}(\abst{i},\abgoal{i}) +
\abminr{i}(\abst{i},\abgoal{i})$).

\begin{lemma}
\label{LEM.ADDBETTER}
Let $\absys$ be any additive abstraction system
and let $t,\goal \in T$ be any states.
If
$\haddtg \ge
\abminc{j}(\abst{j},\abgoal{j}) + \abminr{j}(\abst{j}, \abgoal{j})$
for all $j  \in \{ 1, \dots, k \}$,
then $ \haddtg \geq \hmaxtg $.
\end{lemma}
\begin{proof}
By the definition of $\abopt{i}(\abst{i},\abgoal{i})$,
$\forall j \in \{ 1, \dots, k \}, \abminc{j}(\abst{j},\abgoal{j}) + \abminr{j}(\abst{j}, \abgoal{j}) \ge \abopt{j}(\abst{j},\abgoal{j})$.
Therefore, $\forall j \in \{ 1, \dots, k \}, \haddtg \geq \abminc{j}(\abst{j},\abgoal{j}) + \abminr{j}(\abst{j}, \abgoal{j}) \geq \abopt{j}(\abst{j},\abgoal{j}) \Rightarrow  \haddtg \geq
\max_{ 1 \le i \le k }
\abopt{i}(\abst{i},\abgoal{i}) = \hmaxtg$.
\end{proof}

\begin{lemma}
\label{LEMCLAIM.INFEASIBLE.1}
 For an additive $\absys$ and
path $\gpath \in \gpaths{\gst}{\goal}$
with $\gcost(\gpath) = \sum_{i=1}^k \abminc{i}(\abst{i},\abgoal{i})$,
$ \abcost{j}(\abpath{j}) = \abminc{j}(\abst{j},\abgoal{j})$
for all $j  \in \{ 1, \dots, k \}$.
\end{lemma}

\begin{proof}
Suppose for a contradiction that there exists some $i_1$,
such that
$ \abcost{i_1}(\abpath{i_1}) > \abminc{i_1}(\abst{i_1},\abgoal{i_1}) $.
Then because $\gcost(\gpath) =
\sum_{i=1}^k \abminc{i}(\abst{i},\abgoal{i})$,
there must exist some $i_2$, such that
$ \abcost{i_2}(\abpath{i_2}) < \abminc{i_2}(\abst{i_2},\abgoal{i_2}) $,
which contradicts the definition of $ \abminc{i}$.
Therefore, such an $i_1$ does not exist and
$ \abcost{j}(\abpath{j}) = \abminc{j}(\abst{j},\abgoal{j})$
for all $j  \in \{ 1, \dots, k \}$.
\end{proof}

\begin{lemma}
\label{LEMCLAIM.INFEASIBLE.2}
For an additive $\absys$ and
a path $\gpath \in \gpaths{\gst}{\goal}$
with $\gcost(\gpath) = \sum_{i=1}^k \abminc{i}(\abst{i},\abgoal{i})$,
$ \abresid{i}(\abpath{i}) \ge \abminr{i}(\abst{i},\abgoal{i})$
for all $i  \in \{ 1, \dots, k \}$.
\end{lemma}

\begin{proof}
Following Lemma~\ref{LEMCLAIM.INFEASIBLE.1} and the definition of $\primeset{i}{\abst{i}}{\abgoal{i}}$,
$\abpath{i} \in \primeset{i}{\abst{i}}{\abgoal{i}}$
for all $i  \in \{ 1, \dots, k \}$.
Because $\abminr{i}(\abst{i},\abgoal{i})$ is the smallest residual
cost of paths in $\primeset{i}{\abst{i}}{\abgoal{i}}$, it follows
that $ \abresid{i}(\abpath{i}) \ge \abminr{i}(\abst{i},\abgoal{i})$.
\end{proof}

\begin{lemma}
\label{LEMCLAIM.INFEASIBLE.3}
For an additive $\absys$ and
a path $\gpath \in \gpaths{\gst}{\goal}$
with $\gcost(\gpath) = \sum_{i=1}^k \abminc{i}(\abst{i},\abgoal{i})$,
$\sum_{i=1}^k \abminc{i}(\abst{i},\abgoal{i})
\ge \abminc{j}(\abst{j},\abgoal{j}) + \abminr{j}(\abst{j},\abgoal{j})$
for all $j  \in \{ 1, \dots, k \}$.
\end{lemma}

\begin{proof}
By Lemma \ref{LEM.PathCost},
$\gcost(\gpath) \ge \abcost{j}(\abpath{j}) +\abresid{j}(\abpath{j})$
for all $j  \in \{ 1, \dots, k \}$.
By Lemma~\ref{LEMCLAIM.INFEASIBLE.1}
$\abcost{j}(\abpath{j}) = \abminc{j}(\abst{j},\abgoal{j}) $,
and by Lemma~\ref{LEMCLAIM.INFEASIBLE.2}
$\abresid{j}(\abpath{j}) \ge \abminr{j}(\abst{j},\abgoal{j})$.
Therefore $ \gcost(\gpath) \ge
\abminc{j}(\abst{j},\abgoal{j}) + \abminr{j}(\abst{j},\abgoal{j})$,
and the lemma follows from the premise that
$\gcost(\gpath) = \sum_{i=1}^k \abminc{i}(\abst{i},\abgoal{i})$.
\end{proof}

\begin{lemma}
\label{LEM.INFEASIBLE}
Let $\absys$ be any additive abstraction system
and let $t,\goal \in T$ be any states.
If
$\haddtg <
\abminc{j}(\abst{j},\abgoal{j}) + \abminr{j}(\abst{j}, \abgoal{j})$
for some $j  \in \{ 1, \dots, k \}$,
then  $\haddtg \ne OPT(t,g)$.
\end{lemma}
\begin{proof}
This lemma follows directly as
the contrapositive of Lemma~\ref{LEMCLAIM.INFEASIBLE.3}.
\end{proof}

Lemma \ref{LEM.ADDBETTER} gives a condition under which $\hadd$ is
guaranteed to be at least as large as $\hmax$ for a specific states $t$ and $g$.
If this condition holds for a large fraction of the state space $T$,
one would expect that search using $\hadd$ to be at least as fast as,
and possibly faster than, search using $\hmax$.
This will be seen in the experiments reported in Section \ref{applications}.
The opposite
is not true in general,  i.e., failing
this condition does not imply that $\hmax$ will
result in faster search than $\hadd$.
However, as Lemma \ref{LEM.INFEASIBLE} shows, there is an
interesting consequence when this condition fails for state $t$:
we know that the value returned by $\hadd$ for $t$ is not the
true cost to reach the goal from $t$.
Detecting this is useful because it allows the heuristic value
to be increased without risking it becoming inadmissible.
Section \ref{infeasible} explores this in detail.

\input{section_prev_abs.tex}

\input{section_prev_add.tex}

\input{newApplicationsIntro.tex}

\subsection{TopSpin with Cost-Splitting}
\label{resultsTopSpin}

In the $(N,K)$-TopSpin puzzle (see Figure \ref{TOPSPIN}) there are
$N$ tiles (numbered $1, \dots ,N$) arranged on a circular
track, and two physical movements are possible: (1) the entire set
of tiles may be rotated around the track, and (2) a segment
consisting of $K$ adjacent tiles in the track may be reversed. As in
previous formulations
of this puzzle as a state space \cite{duallookups,HierarchialRevisited,multiplePDB},
we do not represent the first physical
movement as an operator, but instead
designate one of the tiles (tile 1) as a reference tile with the
goal being to get the other tiles in increasing order starting
from this tile (regardless of its position).
The state space therefore has $N$ operators (numbered $1, \dots ,N$),
with operator $a$ reversing the segment of length $K$ starting
at position $a$ relative to the current position of tile 1.
For certain combinations of $N$ and $K$ all possible permutations
can be generated from the standard goal state by these operators, but in general the space consists of connected components  and so not all states are reachable \cite{chen96}. In the experiments in this
section, $K=4$ and $N$ is varied.

\begin{figure}[htb]
\centerline{
\includegraphics[height=1.6in]{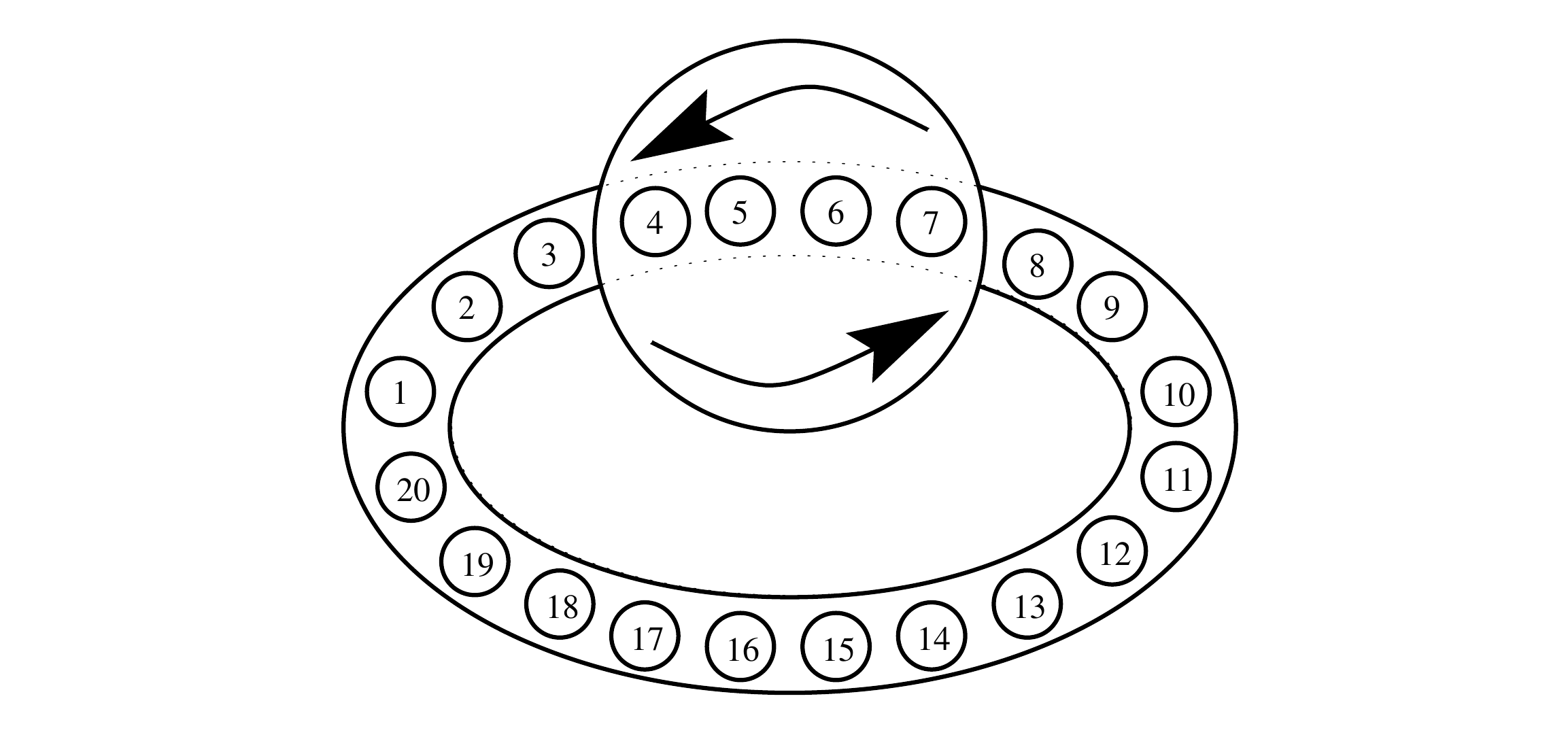}
}
\caption{The TopSpin puzzle.} \label{TOPSPIN}
\end{figure}

The sets of abstractions used in these experiments are described
using a tuple written as $a_1$--$a_2$--$\dots$--$a_M$,  indicating
that the set contains $M$ abstractions,
with tiles $1 \dots (a_1)$ distinguished in the first abstraction,
tiles $(a_1+1) \dots (a_1+a_2)$
distinguished in the second abstraction, and so on.
For example, 6-6-6 denotes a set of three abstractions in
which the distinguished tiles are ($1 \dots 6$), ($7 \dots 12$), and
($13 \dots 18$) respectively.

The experiments compare $\hadd$, the additive use of a set of
abstractions, with $\hmax$, the standard use of the same
abstractions, in which, as described in Section
\ref{heuristicAbstractions}, the full cost of each state transition
is counted in each abstraction and the heuristic returns the maximum
distance to goal returned by the different abstractions.
Cost-splitting is used to define operator costs in the abstract
spaces for $\hadd$. Because $K=4$, each operator moves 4 tiles. If
$b_i$ of these are distinguished tiles when operator $op$ is applied
to state $s_i$ in abstraction $i$, applying $op$ to $s_i$ has a
primary cost of $\frac{b_i}{4}$ in abstraction $i$.

In these experiments the heuristic defined by each
abstraction is stored in a pattern database (PDB).
Each abstraction would normally be used to define its own PDB,
so that a set of $M$ abstractions would require $M$ PDBs.
However, for TopSpin, if two (or more) abstractions have the same number
of distinguished tiles and the distinguished tiles are all adjacent,
one PDB can be used
for all of them by suitably renaming the tiles before doing the PDB lookup.
For the 6-6-6 abstractions, for example, only one PDB is needed, but
three lookups would be done in it, one for each abstraction.
Because the position of tile 1 is effectively fixed, this PDB is $N$ times
smaller than it would normally be.  For example, with $N=18$, the
PDB for the 6-6-6 abstractions contains
$17 \times 16 \times \ldots \times 13$ entries.
The memory needed for each entry in the $\hadd$ PDBs is twice the memory
needed for an entry in the $\hmax$ PDBs because of the need to represent
fractional values.

We ran experiments for the values of $N$ and sets of abstractions shown in
the first two columns of Table \ref{ALLRES}.
Start states were generated by a random walk of 150 moves
from the goal state.  There were 1000, 50 and 20 start states for
$N= 12, 16$ and $18$, respectively.  The average solution length
for these start states is shown in the third column of Table \ref{ALLRES}.
The average number of nodes generated and the average CPU time (in seconds)
for $\textit{IDA}^*$ to solve the given start states is shown in the {\bf Nodes}
and {\bf Time} columns for each of $\hmax$ and $\hadd$.
The {\bf Nodes Ratio} column gives the ratio of {\bf Nodes} using $\hadd$ to
{\bf Nodes} using $\hmax$.  A ratio less than one
(highlighted in bold) indicates
that $\hadd$, the heuristic based on additive abstractions with
cost-splitting,
is superior to $\hmax$, the standard heuristic using the same set of
abstractions.

\begin{table}[htb]
\begin{center}
\begin{tabular}{|l|r|r|r|r|r|r|r|}\hline
    &             & Average &  \multicolumn{2}{|c|}{ }          & \multicolumn{2}{|c|}{$\hadd$ based on}&    \\
$N$ & Abs & Solution & \multicolumn{2}{|c|}{\raisebox{1.5ex}[0pt]{$\hmax$}} & \multicolumn{2}{|c|}{cost-splitting}& Nodes \\
\cline{4-7}
& & Length &   Nodes &  Time    & Nodes   & Time & Ratio \\
\hline \hline
12 & 6-6  &  9.138        & 14,821    & 0.05 & 53,460  & 0.16    & 3.60   \\
12 & 4-4-4 & 9.138      & 269,974   & 1.10 & 346,446   & 1.33  &1.28    \\
12& 3-3-3-3 &  9.138     & 1,762,262   & 8.16 & 1,388,183  & 6.44  &{\bf 0.78} \\
\hline\hline
16 & 8-8 &   14.040 & 1,361,042     & 3.42 & 2,137,740    &  4.74  & 1.57        \\
16 &4-4-4-4 & 14.040    & 4,494,414,929     & 13,575.00 &  251,946,069   &  851.00  &  {\bf 0.056}  \\
\hline \hline
18 & 9-9   & 17.000 & 38,646,344 & 165.42& 21,285,298  &  91.76    & {\bf 0.55}     \\
18 & 6-6-6 & 17.000   & 18,438,031,512  &   108,155.00  & 879,249,695    & 4,713.00  & {\bf 0.04}\\

\hline
\end{tabular}
\end{center}
\caption{($N,4$)-TopSpin results using cost-splitting.} \label{ALLRES}
\end{table}

When $N=12$ and $N=16$ the best performance is achieved by $\hmax$
based on a pair of abstractions each having $\frac{N}{2}$
distinguished tiles.  As $N$ increases the advantage of $\hmax$
decreases and, when $N=18$, $\hadd$ outperforms $\hmax$ for all
abstractions used. Moreover, even for the smaller values of $N$
$\hadd$ outperforms $\hmax$ when a set of four abstractions with
$\frac{N}{4}$ distinguished tiles each is used. This is important
because as $N$ increases, memory limitations will preclude using
abstractions with $\frac{N}{2}$ distinguished tiles and the only option
will be to use more abstractions with fewer distinguished tiles
each.  The results in Table \ref{ALLRES} show that $\hadd$ will be
the method of choice in this situation.


\subsection{The Pancake Puzzle with Location-based Costs}
\label{fixreference}

In this section, we present the experimental results on the
17-Pancake puzzle using location-based costs. The same notation as
in the previous section is used to denote sets of abstractions,
{\em e.g.}\ 5-6-6 denotes a set of three abstractions, with the first
having tiles ($0 \dots 4$) as its distinguished tiles, the second
having tiles ($5 \dots 10$) as its distinguished tiles, and the
third having tiles ($11 \dots 16$) as its distinguished tiles. Also
as before, the heuristic for each abstraction is precomputed and
stored in a pattern database (PDB).  Unlike TopSpin, there are no
symmetries in the Pancake puzzle that enable different abstractions
to make use of the same PDB, so a set of $M$ abstractions for the
Pancake puzzle requires $M$ different PDBs.

Additive abstractions are defined using the location-based method with
just one reference location, the leftmost position.
This position was chosen because the tile in this position changes whenever
any operator is applied to any state in the original state space.
This means that every edge cost in the original space will be
fully counted in some abstract space as long as each tile is
a distinguished tile in some abstraction.
As before, we use $\hadd$ to denote the heuristic defined by adding
the values returned by the individual additive abstractions.

Our first experiment compares $\textit{IDA}^*$ using $\hadd$ with the best
results known for the 17-Pancake puzzle \cite{dualsearch} (shown in Table
\ref{tbl-17pancakeresults-1}), which were obtained using a single
abstraction having the rightmost seven tiles (10--16) as its distinguished
tiles and an advanced search technique called Dual $\textit{IDA}^*$
($\textit{DIDA}^*$).\footnote{In particular, $\textit{DIDA}^*$ with the
``{\em jump if larger}" (JIL) policy and the bidirectional pathmax method
(BPMX) to propagate the inconsistent heuristic values that arise during
dual search. \citeA{dualsearch} provided more details. BPMX was first introduced
by Felner et al. \citeyear{duallookups}.} $\textit{DIDA}^*$ is an extension of
$\textit{IDA}^*$ that exploits the fact that, when states are permutations
of tiles as in the Pancake puzzle, each state $s$ has an easily computable
``dual state" $s^d$ with the special property that inverses of paths from
$s$ to the goal are paths from $s^d$ to the goal.  If paths and their
inverses cost the same, $\textit{DIDA}^*$ defines the heuristic value for
state $s$ as the maximum of $h(s)$ and $h(s^d)$, and sometimes will decide
to search for a least-cost path from $s^d$ to goal when it is looking for
a path from $s$ to goal.

The results of this experiment are shown in the top three rows of
Table \ref{tbl-17pancakeresults}. The {\bf Algorithm} column
indicates the heuristic search algorithm. The {\bf Abs} column shows
the set of abstractions used to generate heuristics.
The {\bf Nodes} column shows the average number of nodes generated in solving
1000 randomly generated start states.  These start states have an
average solution length of 15.77.
The {\bf Time} column gives the average number of CPU seconds
needed to solve these start states
on an AMD Athlon(tm) 64 Processor 3700+ with 2.4 GHz clock rate and 1GB memory.
The {\bf Memory} column indicates the total size of each set of PDBs.

\begin{table}[htb]
\begin{center}
\begin{tabular}{|l|r|c|r|r|r|r|}\hline
&   &   & Average &   \multicolumn{3}{|c|}{$h$ based on} \\
$N$ &Algorithm & Abs & Solution & \multicolumn{3}{|c|}{a single large PDB} \\
\cline{5-7}
& & & Length &   Nodes &  Time    & Memory    \\
\hline \hline
17 & $\textit{DIDA}^*$&rightmost-7 &  15.77  & 124,198,462 &37.713&98,017,920   \\
\hline
\end{tabular}
\end{center}
\caption{The best results known for the 17-Pancake puzzle \cite{dualsearch}, which were obtained using a single abstraction having the rightmost seven tiles ($10-16$) as its distinguished tiles and an advanced search technique called Dual $\textit{IDA}^*$ ($\textit{DIDA}^*$).} \label{tbl-17pancakeresults-1}
\end{table}

\begin{table}[htb]
\begin{center}
\begin{tabular}{|l|r|c|r|r|r|r|}\hline
&   &   & Average &   \multicolumn{3}{|c|}{$\hadd$ based on} \\
$N$ &Algorithm & Abs & Solution & \multicolumn{3}{|c|}{Location-based Costs} \\
\cline{5-7}
& & & Length &   Nodes &  Time    & Memory    \\
\hline \hline
17 & $\textit{IDA}^*$&4-4-4-5 &  15.77  & 14,610,039  & 4.302 & 913,920   \\
17 & $\textit{IDA}^*$&5-6-6 &  15.77  & 1,064,108  & 0.342&18,564,000  \\
17 &$\textit{IDA}^*$&3-7-7 &  15.77  & 1,061,383 & 0.383&196,039,920 \\
\hline
17 & $\textit{DIDA}^*$&4-4-4-5 &  15.77  & 368,925&0.195 &  913,920   \\
17 & $\textit{DIDA}^*$&5-6-6 &  15.77  & 44,618 &0.028 &  18,564,000  \\
17 & $\textit{DIDA}^*$&3-7-7 &  15.77  & 37,155 &0.026 & 196,039,920  \\
\hline
\end{tabular}
\end{center}
\caption{17-Pancake puzzle results using $\hadd$ based on location-based costs. } \label{tbl-17pancakeresults}
\end{table}


Clearly, the use of $\hadd$ based on location-based costs results in a very significant
reduction in nodes generated compared to using a single large PDB,
even when the latter has the advantage of being used by a more
sophisticated search algorithm. Note that the total memory needed
for the 4-4-4-5 PDBs is only one percent of the memory needed for
the rightmost-7 PDB, and yet $\textit{IDA}^*$ with 4-4-4-5 generates 8.5 times
fewer nodes than $\textit{DIDA}^*$with the rightmost-7 PDB. Getting excellent
search performance from a very small PDB is especially important in
situations where the cost of computing the PDBs must be taken into
account in addition to the cost of problem-solving
\cite{HierarchialRevisited}.

The memory requirements increase significantly when abstractions
contain more distinguished tiles, but in this experiment the
improvement of the running time does not increase accordingly. For
example, the 3-7-7 PDBs use ten times more memory than the 5-6-6
PDBs, but the running time is almost the same. This is because the
5-6-6 PDBs are so accurate there is little room to improve them. The
average heuristic value on the start states using the 5-6-6 PDBs is
13.594, only 2.2 less than the actual average solution length. The
average heuristic value using the 3-7-7 PDBs is only slightly higher
(13.628).

The last three rows in Table \ref{tbl-17pancakeresults}
show the results when $\hadd$ with location-based costs is used in
conjunction with $\textit{DIDA}^*$.
These results
show that combining our additive abstractions with state-of-the-art search
techniques results in further significant reductions in nodes generated
and CPU time.
For example, the 5-6-6 PDBs
use only 1/5 of the memory of the rightmost-7 PDB
but reduce the number of nodes generated by $\textit{DIDA}^*$ by a factor of 2783
and the CPU time by a factor of 1347.

To compare $\hadd$ to $\hmax$ we ran plain $\textit{IDA}^*$ with $\hmax$
on the same 1000 start states, with a time limit for each start state
ten times greater than the time needed to solve the start state using
$\hadd$.
With this time limit only 63 of the 1000 start states could be solved
with $\hmax$ using the 3-7-7 abstraction, only 5 could be solved
with $\hmax$ using the 5-6-6 abstraction, and only 3 could be solved
with $\hmax$ using the 4-4-4-5 abstraction.
To determine if $\hadd$'s superiority over $\hmax$ for location-based costs
on this puzzle could have been predicted using Lemma \ref{LEM.ADDBETTER},
we generated 100 million random
17-Pancake puzzle states and tested how many satisfied the requirements
of Lemma \ref{LEM.ADDBETTER}.
Over 98\% of the states satisfied those requirements
for the 3-7-7 abstraction, and over 99.8\% of
the states satisfied its requirements for the 5-6-6 and 4-4-4-5 abstractions.

\section{Negative Results}
\label{negative}
Not all of our experiments yielded positive results.  Here we explore some trials where our additive approaches did not perform as well.  By examining some of these cases closely, we shed light on the conditions which might indicate when these approaches will be useful.

\subsection{TopSpin with Location-Based Costs}
\label{TopSpin with Location-Based}
In this experiment, we used the 6-6-6 abstraction of
$(18,4)$-TopSpin as in Section \ref{resultsTopSpin} but with
location-based costs instead of cost-splitting. The primary cost of
operator $a$, the operator that reverses the segment consisting of
locations $a$ to $a+3$ ($modulo$ $18$), is $1$ in abstract space $i$
if the tile in location $a$ before the operator is applied is
distinguished according to abstraction $\psi_i$ and $0$ otherwise.

This definition of costs was disastrous, resulting in
$\abminc{i}(\abst{i},\abgoal{i})=0$ for all abstract states in all
abstractions. In other words, in finding a least-cost path
it was never necessary to use operator
$a$ when there was a distinguished tile in location $a$. It was
always possible to move towards the goal by applying another
operator, $a'$, with a primary cost of $0$.
To
illustrate how this is possible, consider state
\begin{tabular}{|c|c|c|c|c|c|c|}
\hline           0&4&5&6&3&2&1\\
\hline
\end{tabular}
of $(7,4)$-$TopSpin$.
This state can be transformed into the goal in a single move: the
operator that reverses the four tiles starting with tile $3$ produces the state
\begin{tabular}{|c|c|c|c|c|c|c|}
\hline           3&4&5&6&0&1&2\\
\hline
\end{tabular}
which is equal to the goal state when it is cyclically shifted to
put $0$ into the leftmost position. With the 4-3 abstraction this
move has a primary cost of $0$ in the abstract space based on tiles
$4...6$, but it would have a primary cost of $1$ in the abstract
space based on tiles $0...3$ (because tile $3$ is in the leftmost
location changed by the operator). However the following sequence
maps tiles $0...3$ to their goal locations and has a primary cost of
$0$ in this abstract space (because a ``don't care'' tile is always moved
from the reference location):
\begin{center}
\begin{tabular}{|c|c|c|c|c|c|c|c|}
\hline           0&*&*&*&3&2&1\\
\hline
\end{tabular}

\begin{tabular}{|c|c|c|c|c|c|c|c|}
\hline           0&*&*&1&2&3&*\\
\hline
\end{tabular}

\begin{tabular}{|c|c|c|c|c|c|c|c|}
\hline           0&*&3&2&1&*&*\\
\hline
\end{tabular}

\begin{tabular}{|c|c|c|c|c|c|c|c|}
\hline           0&1&2&3&*&*&*\\
\hline
\end{tabular}

\end{center}

\subsection{Rubik's Cube}

The success of cost-splitting on (18,4)-TopSpin suggested it might also
provide an improved heuristic for Rubik's Cube, which can be viewed
as a 3-dimensional version of (20,8)-TopSpin.
We used the standard method of partitioning the cubies to create
three abstractions, one based on the 8 corner cubies, and the others
based on 6 edge cubies each.
The standard heuristic based on this partitioning, $\hmax$, expanded
approximately three times fewer nodes than $\hadd$ based on this
partitioning and primary costs defined by cost-splitting.
The result was similar whether the 24 symmetries of Rubik's Cube
were used to define multiple heuristic lookups or not.

We believe the reason for cost-splitting working well for (18,4)-TopSpin but
not Rubik's Cube is that an operator in Rubik's Cube moves more cubies
than the number of tiles moved by an operator in (18,4)-TopSpin.
To test if operators
moving more tiles reduces the effectiveness of cost-splitting
we solved 1000 instances of (12,$K$)-TopSpin for various values of $K$, all
using the 3-3-3-3 abstraction.
The results are shown in Table \ref{topspin12}.
The {\bf Nodes Ratio} column
gives the ratio of {\bf Nodes} using $\hadd$ to {\bf Nodes} using $\hmax$.
A ratio less than one (highlighted in bold) indicates that
$\hadd$ is superior to $\hmax$.
The results  clearly show
that $\hadd$ based on cost-splitting is superior to $\hmax$ for small
$K$ and steadily
loses its advantage as $K$ increases.
The same phenomenon can also be seen in Table \ref{ALLRES}, where increasing
$N$ relative to $K$ increases the effectiveness
of additive heuristics based on cost-splitting.

\begin{table}[htb]
\begin{center}
\begin{tabular}{|r|r|r|r|r|r|}\hline
 &  \multicolumn{2}{|c|}{ }          & \multicolumn{2}{|c|}{$\hadd$ based on}&  \\

 $K$ & \multicolumn{2}{|c|}{\raisebox{1.5ex}[0pt]{$\hmax$}} & \multicolumn{2}{|c|}{cost-splitting}& Nodes \\
\cline{2-5}
 & Nodes &  Time    & Nodes   & Time & Ratio \\
\hline \hline
3 & 486,515    &  2.206 & 207,479    & 0.952        & {\bf 0.42}   \\
4 & 1,762,262   & 8.164 & 1,388,183  & 6.437        & {\bf 0.78}    \\
5 & 8,978   &     0.043 & 20,096     & 0.095        &  2.23 \\
6 & 193,335,181 & 901.000 & 2,459,204,715 & 11,457.000  & 12.72  \\
\hline
\end{tabular}
\end{center}
\caption{($12,K$)-TopSpin results using cost-splitting.} \label{topspin12}
\end{table}

We also investigated location-based costs for Rubik's Cube.
The cubies were partitioned into four groups, each containing
three edge cubies and two corner cubies, and an abstraction was defined
using each group.
Two diagonally opposite corner positions were used as the reference
locations (as noted above, each Rubik's Cube operator changes exactly
one of these locations).
The resulting $\hadd$ heuristic was so weak we could not solve random
instances of the puzzle with it.

\subsection{The Pancake Puzzle with Cost-Splitting}
\label{pancake-costsplitting}

Table \ref{tbl-13pancakeresults} compares
$\hadd$ and $\hmax$ on the 13-Pancake puzzle
when costs are defined using cost-splitting.
The memory is greater for $\hadd$ than $\hmax$ because
the fractional entries that cost-splitting produces require
more bits per entry than the small integer values stored in
the $\hmax$ PDB.
In terms of both run-time and number of nodes generated,
$\hadd$ is inferior to $\hmax$ for these costs,
the opposite of what was seen in Section \ref{fixreference}
using location-based costs.



\begin{table}[htb]
\begin{center}
\begin{tabular}{|l|r|r|r|r|r|r|}\hline
    &             & Average &  \multicolumn{2}{|c|}{ }          & \multicolumn{2}{|c|}{$\hadd$ based on}  \\
$N$ & Abs & Solution & \multicolumn{2}{|c|}{\raisebox{1.5ex}[0pt]{$\hmax$}} & \multicolumn{2}{|c|}{costing-splitting} \\
\cline{4-7}
& & Length &   Nodes &  Time    & Nodes   & Time   \\
\hline \hline
13 & 6-7  &  11.791       & 166,479    & 0.0466 & 1,218,903  & 0.3622        \\
\hline
\end{tabular}
\end{center}
\caption{$\hadd$ vs. $\hmax$ on the 13-Pancake puzzle.} \label{tbl-13pancakeresults}
\end{table}

Cost-splitting, as we have defined it for the Pancake puzzle,
adversely affects $\hadd$ because it enables
each individual abstraction to get artificially low estimates
of the cost of solving its distinguished tiles by increasing the number of
``don't care" tiles that are moved.
For example, with cost-splitting
the least-cost sequence of operators to get tile ``0"
into its goal position from abstract state
\begin{tabular}{|c|c|c|c|c|}
\hline           *&0&*&*&*\\
\hline
\end{tabular}
is not the obvious single move of reversing the first two positions.
That move costs $\frac{1}{2}$, whereas the
2-move sequence that reverses the entire state
and then reverses the first four positions costs only $\frac{1}{5}+\frac{1}{4}$.

\vspace{0.05in}
As a specific example, consider state
\begin{tabular}{|c|c|c|c|c|c|c|c|c|c|c|c|}
\hline           7&4&5&6&3&8&0&10&9&2&1&11\\
\hline
\end{tabular}
of the 12-Pancake puzzle. Using the 6-6 abstractions,
the minimum number of moves
to get tiles 0--5 into their goal
positions is 8,
and for 6--11 it is 7, where in each case we ignore the final locations of the other tiles.
Thus, $\hmax$ is 8.
By contrast, $\hadd$ is $6.918$, which is less than even the smaller
of the two numbers used to define $\hmax$.
The two move sequences whose costs are added to compute
$\hadd$ for this state each have slightly
more moves than the corresponding sequences
on which $\hmax$ is based (10 and 9 compared to 8 and 7),
but involve more than twice
as many ``don't care" tiles (45 and 44 compared to 11 and 17) and
so are less costly.

There is hope that this pathological situation can be detected, at least
sometimes, by inspecting the residual costs.
If the residual costs are defined to be complementary to the primary costs
({\em i.e.}\ $R_i(\pi_i)=C(\pi)-C_i(\pi_i)$), as we have done, then decreasing
the primary cost increases the residual cost. If the residual cost
is sufficiently
large in one of the abstract spaces the conditions of Lemma \ref{LEM.INFEASIBLE}
will be satisfied, signalling that the value returned by $\hadd$ is
provably too low.
This is the subject of the next section, on ``infeasibility".

\section{Infeasible Heuristic Values}
\label{infeasible}

This section describes a way to increase the heuristic values
defined by additive abstractions in some circumstances.
The key to the approach is to identify ``infeasible" values---ones
that cannot possibly be the optimal solution cost.
Once identified the infeasible values can be increased to give a better estimate
of the solution cost.
An example of infeasibility occurs with the Manhattan Distance ($\textit{MD}$)
heuristic for the sliding tile puzzle.
It is well-known that the parity of $\textit{MD}(t)$ is the same as the parity
of the optimal solution cost for state $t$.
If some other heuristic for the sliding tile puzzle returns a value
of the opposite parity, it can safely be increased until it has the correct
parity.
This example relies on specific properties of the $\textit{MD}$ heuristic and the
puzzle.
Lemma \ref{LEM.INFEASIBLE} gives a problem-independent method for testing
infeasibility, and that is what we will use.

To illustrate how infeasibility can be detected using Lemma
\ref{LEM.INFEASIBLE} consider the example in Figure
\ref{fig-infeasible}. The solution to the abstract problem shown in
the middle part of the figure requires 9 distinguished moves, so
$C^*_1(A_1)=9$. The abstract paths that solve the problem with 9
distinguished moves require, at a minimum, 9 ``don't care" moves, so
$R^*_1(A_1)=9$. A similar calculation for the abstract space on the
right of the figure yields $C^*_2(A_2)=5$ and $R^*_2(A_2)=7$. The
value of $\hadda{A}{\goal}$ is therefore $C^*_1(A_1)+C^*_2(A_2)=9+5=14$.
This value is based on the assumption that there is a path in the
original space that makes $C^*_1(A_1)=9$ moves of tiles 1, 3, 5, and
7, and $C^*_2(A_2)=5$ moves of the other tiles. However, the value
of $R^*_1(A_1)$ tells us that any path that uses only $9$ moves of
tiles 1, 3, 5, and 7 to put them into their goal locations must make
at least $9$ moves of the other tiles, it cannot possibly make just
$5$ moves. Therefore there does not exist a solution costing as
little as $C^*_1(A_1)+C^*_2(A_2)=14$.

\label{infeasibleresult}

To illustrate the potential of this method for improving additive
heuristics, Table \ref{fig-15slidingpuzzle} shows the average
results of $\textit{IDA}^*$ solving 1000 test instances of the 15-puzzle using
two different tile partitionings (shown in Figure
\ref{fig-partition}) and costs defined by the method described in
Section \ref{additivePDB}. These additive heuristics have the same
parity property as Manhattan Distance, so when infeasibility is
detected $2$ can be added to the value. The {\bf $\hadd$} columns
show the average heuristic value of the 1000 start states. As can
be seen infeasibility checking increases the initial heuristic value
by over 0.5 and reduces the number of nodes generated and the CPU
time by over a factor of $2$. However, there is a space penalty for
this improvement, because the $R^*$ values must be stored in the
pattern database in addition to the normal $C^*$ values. This
doubles the amount of memory required, and it is not clear if
storing $R^*$ is the best way to use this extra memory. This
experiment merely shows that infeasibility checking is one way to
use extra memory to speed up search for some problems.

\begin{figure}[!ht]
\centerline{
\includegraphics[width=3.50cm]{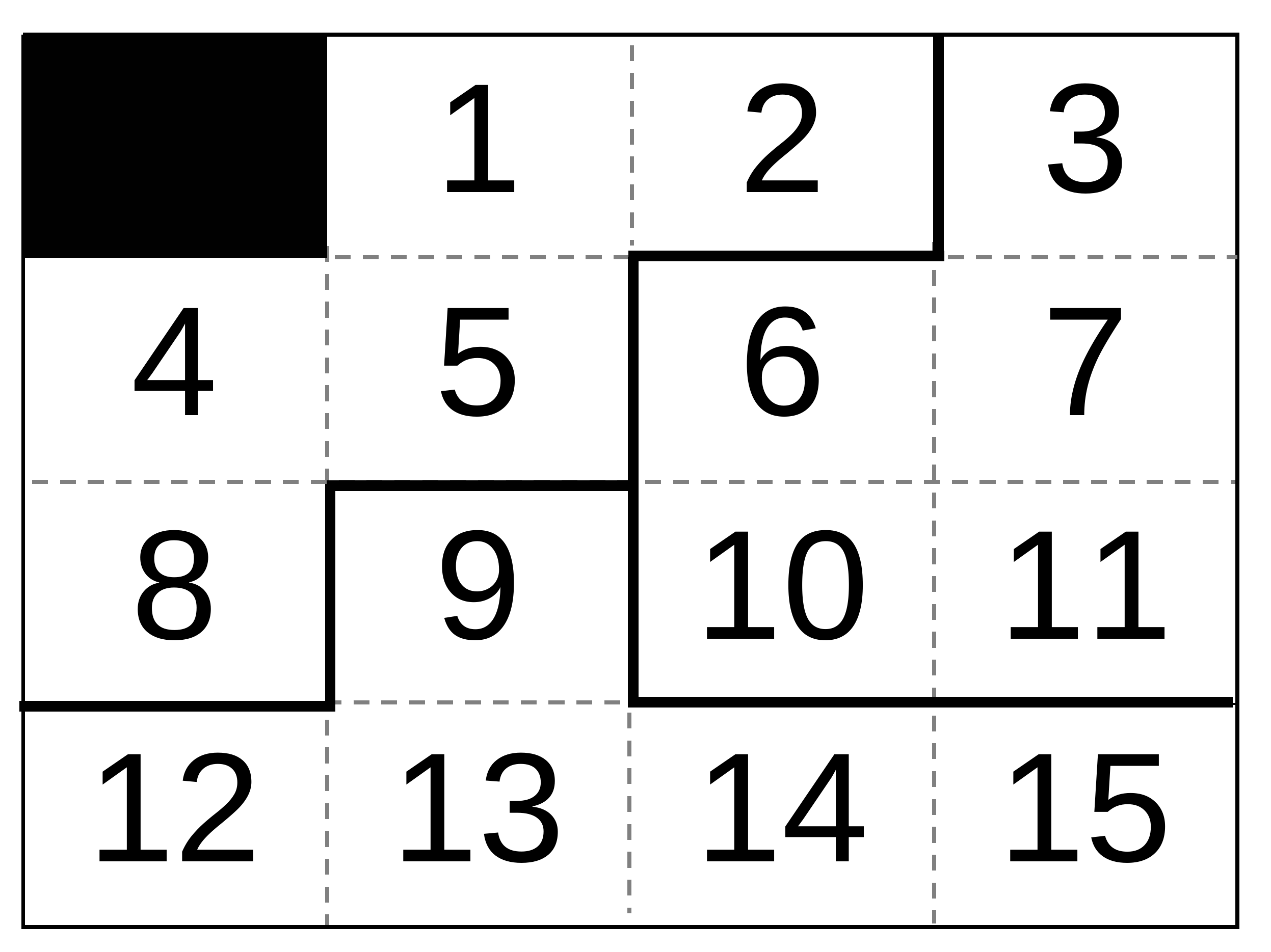}
\hspace{1cm}
\includegraphics[width=3.50cm]{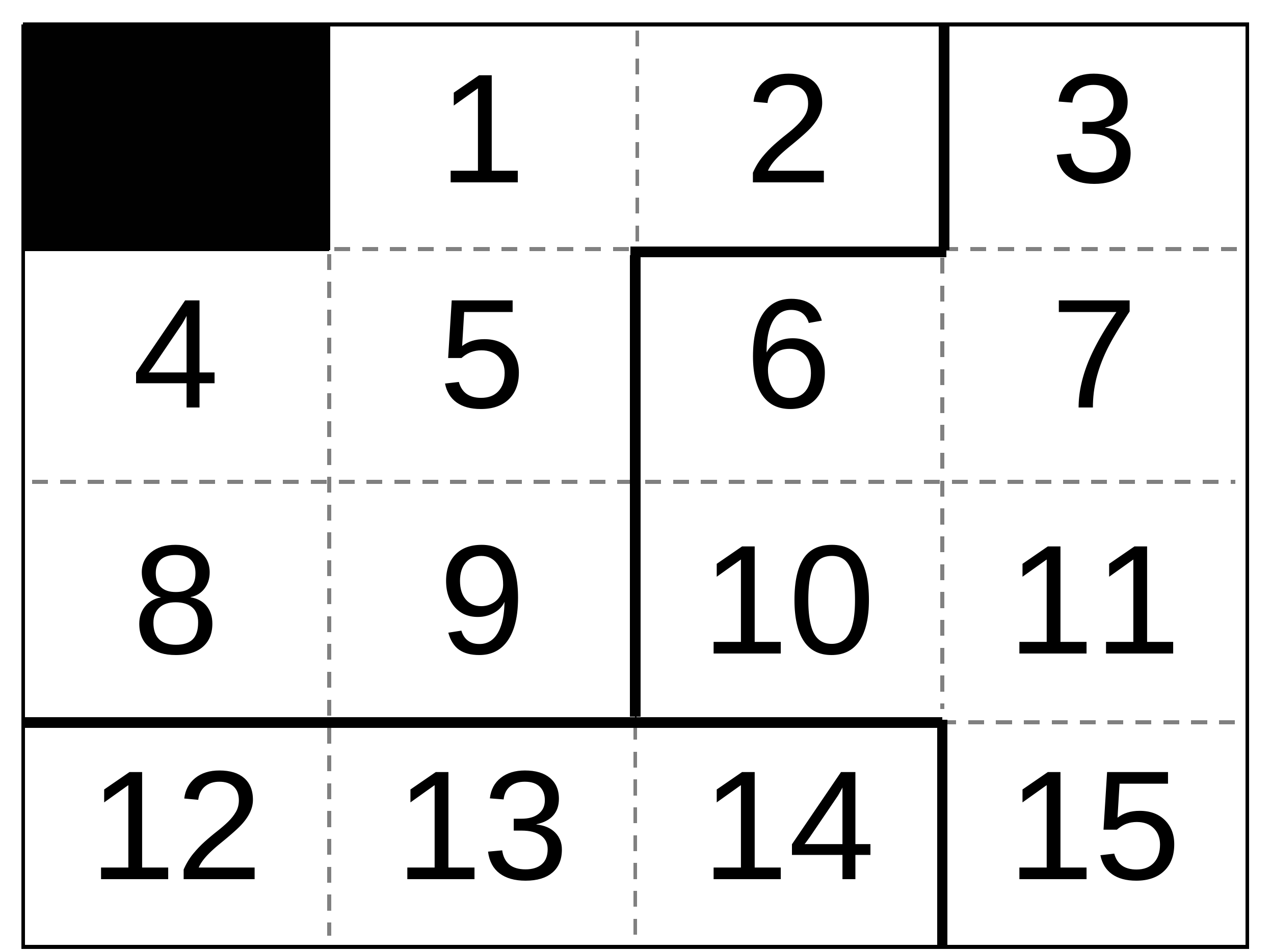}
}
\caption{Different tile partitionings for the 15-puzzle (left: 5-5-5; right: 6-6-3).}
\label{fig-partition}
\end{figure}

\begin{table}[htb]
\begin{center}
\begin{tabular}{|l|r|r|r|r|r|r|r|r|}\hline
    &             & Average  & \multicolumn{6}{|c|}{$\hadd$ based on zero-one cost-splitting}  \\
\cline {4-9}
$N$ & Abs & Solution & \multicolumn{3}{|c|}{No Infeasibility Check}& \multicolumn{3}{|c|}{Infeasibility Check} \\
\cline{4-9}
& & Length & $\hadd$ &  Nodes &  Time & $\hadd$   & Nodes   & Time  \\
\hline \hline
15 & 5-5-5  & 52.522     &41.56  &3,186,654&0.642  &42.10   &1,453,358&0.312     \\
15 & 6-6-3  &  52.522        &42.13 &1,858,899&0.379 &42.78   &784,145&0.171    \\
\hline
\end{tabular}
\end{center}
\caption{The effect of infeasibility checking on the 15-puzzle.} \label{fig-15slidingpuzzle}
\end{table}


Slightly stronger results were obtained for the $(N,4)$-TopSpin puzzle
with costs defined by
cost-splitting, as described in Section \ref{resultsTopSpin}.
The ``No Infeasibility Check" columns
in Table \ref{tbl-TopSpinCheck} are the same as the ``$\hadd$ based on
cost-splitting" columns of the corresponding rows in Table \ref{ALLRES}.
Comparing these to the ``Infeasibility Check" columns shows that in most
cases infeasibility checking reduces the number of nodes generated and
the CPU time by roughly a factor of 2.

When location-based costs are used with TopSpin infeasibility
checking adds one to the heuristic value of almost every state.
However, this simply means that most states have a heuristic value of 1
instead of 0
(recall the discussion in Section \ref{TopSpin with Location-Based}),
which is still a very poor heuristic.

\begin{table}[htb]
\begin{center}
\begin{tabular}{|l|r|r|r|r|r|r|}\hline
    &             & Average  & \multicolumn{4}{|c|}{$\hadd$ based on costing-splitting}  \\
\cline {4-7}
$N$ & Abs & Solution & \multicolumn{2}{|c|}{No Infeasibility Check}& \multicolumn{2}{|c|}{Infeasibility Check} \\
\cline{4-7}
& & Length &   Nodes &  Time    & Nodes   & Time  \\
\hline \hline
12 & 6-6  &  9.138   &53,460  & 0.16  & 20,229  & 0.07    \\
12 & 4-4-4  &  9.138   & 346,446   & 1.33& 174,293   & 0.62    \\
12 & 3-3-3-3  &  9.138& 1,388,183  & 6.44 & 1,078,853  & 4.90    \\
\hline\hline
16 & 8-8  &  14.040  & 2,137,740    &  4.74 & 705,790      &  1.80  \\
16 & 4-4-4-4  &  14.040  &  251,946,069   &  851.00 &  203,213,736   &  772.04  \\
\hline \hline
18 & 6-6-6  &  17.000 &879,249,695    & 4,713.00 & 508,851,444    & 2,846.52    \\
\hline
\end{tabular}
\end{center}
\caption{The effect of infeasibility checking on ($N,4$)-TopSpin using cost-splitting.} \label{tbl-TopSpinCheck}
\end{table}

Infeasibility checking produces almost no benefit for the 17-Pancake puzzle
with location-based costs because the conditions of
Lemma \ref{LEM.INFEASIBLE} are almost never satisfied.
The experiment discussed at the end of Section \ref{fixreference}
showed that fewer than 2\% of the states
satisfy the conditions of Lemma \ref{LEM.INFEASIBLE}
for the 3-7-7 abstraction, and fewer than 0.2\% of the states
satisfy the conditions of Lemma \ref{LEM.INFEASIBLE}
for the 5-6-6 and 4-4-4-5 abstractions.


Infeasibility checking for the 13-Pancake puzzle
with cost-splitting also produces very little benefit, but
for a different reason.
For example, Table \ref{tbl-13pancak-infeasible} shows the effect
of infeasibility checking on the 13-Pancake puzzle;
the results shown are averages over 1000 start states.
Cost-splitting in this state space produces fractional edge costs
that are multiples of
$\frac{1}{360360}$
(360360 is the Least Common Multiple of
the integers from 1 to 13), and therefore if infeasibility is detected
the amount added is
$\frac{1}{360360}$.
But recall that $\hadd$, with cost-splitting, is defined as the ceiling
of $\sum_{i=1}^k \abcost{i}(\abpr{i})$. The value of $\hadd$ will therefore
be the same, whether $\frac{1}{360360}$ is added or not, unless the sum
of the $\abcost{i}(\abpr{i})$ is exactly an integer.
As Table \ref{tbl-13pancak-infeasible} shows, this does happen but only
rarely.

\begin{table}[htb]
\begin{center}
\begin{tabular}{|l|r|r|r|r|r|r|}\hline
    &             & Average  & \multicolumn{4}{|c|}{$\hadd$ based on costing-splitting}  \\
\cline {4-7}
$N$ & Abs & Solution & \multicolumn{2}{|c|}{No Infeasibility Check}& \multicolumn{2}{|c|}{Infeasibility Check} \\
\cline{4-7}
& & Length &   Nodes &  Time    & Nodes   & Time  \\
\hline \hline
13 & 6-7  &  11.791       & 1,218,903  & 0.3622    &1,218,789&0.4453    \\
\hline
\end{tabular}
\end{center}
\caption{The effect of infeasibility checking on the 13-Pancake puzzle using cost-splitting. } \label{tbl-13pancak-infeasible}
\end{table}



\section{Conclusions}
\label{sec-conclusions}
In this paper we have presented a formal, general definition
of additive abstractions that removes the restrictions
of most previous definitions, thereby enabling additive abstractions to be
defined for any state space. We have proven that heuristics
based on additive abstractions are consistent as well as admissible.
Our definition formalizes
the intuitive idea that abstractions will be additive
provided the cost of each operator is divided among the
abstract spaces, and we have presented two specific, practical methods
for defining abstract costs, cost-splitting and location-based costs.
These methods were applied to three standard state spaces that
did not have additive abstractions according to previous
definitions: TopSpin, Rubik's Cube, and the Pancake puzzle.
Additive abstractions using cost-splitting reduce
search time substantially for (18,4)-TopSpin and additive abstractions
using location-based costs reduce search time for the 17-Pancake
puzzle by three orders of magnitude over the state of the art.
We also report negative results, for example on Rubik's Cube, demonstrating
that additive abstractions are not always superior to the standard,
maximum-based method for combining multiple abstractions.

A distinctive feature of our definition is that each edge in
an abstract space has two costs instead of just one.
This was inspired by previous definitions treating ``distinguished"
moves differently than ``don't care" moves in calculating least-cost
abstract paths.
Formalizing this idea with two costs per edge has enabled us
to develop
a way of testing if the heuristic value returned by
additive abstractions is provably too low (``infeasible").
This test produced no speedup when applied to the Pancake puzzle, but
roughly halved the search time for the 15-puzzle and in most of
our experiments with TopSpin.


\section{Acknowledgments}
This research was supported in part by funding from Canada's Natural
Sciences and Engineering Research Council (NSERC). Sandra Zilles and
Jonathan Schaeffer suggested useful improvements to drafts of this paper.
This research was also supported by the Israel Science Foundation (ISF)
under grant number 728/06 to Ariel Felner.



\vskip 0.2in
\bibliographystyle{theapa}
\bibliography{YFADD}

\end{document}

%% file: macros.tex
\newcommand {\qed} {\raisebox{3 pt}{\fbox{}}}

\newtheorem{lemma}{Lemma}[section]

\newenvironment{proof}{\noindent{\bf Proof:}}{
$\qed$
\medskip
}

\newtheorem{definition}{Definition}[section]

\newcommand{\tuple}[1]{\langle #1 \rangle} 

\newcommand {\setdelim}        {\; | \;}

\newcommand{\nat}{{\mathcal{N}}} 

\newcommand{\absset}{\aleph} 
\newcommand{\abmap}[1]{\psi_{#1}} 
\newcommand{\gspace}{{\mathcal{S}}} 
\newcommand{\abspace}[1]{{\mathcal{A}}_{#1}} 
\newcommand{\absys}{\tuple{ \gspace, \absset}} 

\newcommand{\gstates}{T} 
\newcommand{\abstates}[1]{\gstates_{#1}} 
\newcommand{\gst}{t} 
\newcommand{\goal}{g} 
\newcommand{\abst}[1]{\gst_{#1}} 
\newcommand{\abgoal}[1]{\goal_{#1}} 
\newcommand{\gstj}[1]{\gst^{#1}} 
\newcommand{\abstij}[2]{\gst^{#2}_{#1}} 
			
\newcommand{\gabst}{u}  
\newcommand{\gabsg}{v} 
			
\newcommand{\gpairs}{\Pi}  
\newcommand{\abpairs}[1]{\gpairs_{#1}} 
\newcommand{\gpr}{\pi} 
\newcommand{\abpr}[1]{\gpr_{#1}} 
\newcommand{\gprj}[1]{\gpr^{#1}} 
\newcommand{\abprij}[2]{\gpr^{#2}_{#1}} 

\newcommand{\gabpr}{\rho}  

\newcommand{\gcost}{C} 
\newcommand{\gresid}{R} 
\newcommand{\gopt}{{\mbox{\rm OPT}}} 
\newcommand{\gminc}{\gcost^{*}} 
\newcommand{\gminr}{\gresid^{*}} 
\newcommand{\abcost}[1]{\gcost_{#1}} 
\newcommand{\abresid}[1]{\gresid_{#1}} 
\newcommand{\abopt}[1]{\gopt_{#1}} 
\newcommand{\abminc}[1]{\gminc_{#1}} 
\newcommand{\abminr}[1]{\gminr_{#1}} 

\newcommand{\htg}{h(t,g)}
\newcommand{\hi}[2]{h_i(#1,#2)}
\newcommand{\hitg}{\hi{\gst}{\goal}}
\newcommand{\hmax}{h_{max}}
\newcommand{\hmaxa}[2]{\hmax(#1,#2)}  
\newcommand{\hmaxtg}{\hmaxa{\gst}{\goal}} 
\newcommand{\hadd}{h_{add}}
\newcommand{\hadda}[2]{\hadd(#1, #2)} 
\newcommand{\haddtg}{\hadda{\gst}{\goal}}

\newcommand{\primeset}[3]{\vec{P}_{#1}({#2},{#3})} 
\newcommand{\gpath}{\vec{p}}   
\newcommand{\abpath}[1]{\abmap{#1}(\gpath)} 
\newcommand{\abpi}[1]{\gpath_{#1}} 

\newcommand{\gabpath}{\vec{q}} 

\newcommand{\gpaths}[2]{Paths(\gspace, #1,#2)} 
\newcommand{\abpathset}[3]{Paths(\abspace{#1},#2, #3)} 

%% file: section_prev_abs.tex
\subsection{Relation to Previous Work}

The aim of the preceding formal definitions is to identify
fundamental properties that guarantee that abstractions will give rise
to admissible, consistent heuristics.
We have shown that the following two conditions
guarantee that the heuristic defined by an abstraction
is admissible and consistent
\begin{eqnarray*}
(P1) &  & \hspace{0.7cm} \forall (u,v) \in \gpairs, (\abmap{i}(u),\abmap{i}(v)) \in \gpairs_{i}  \label{P1}\\
(P2) &  & \hspace{0.7cm} \forall \gpr \in \gpairs, \gcost(\gpr) \ge \abcost{i}(\abpr{i}) + \abresid{i}(\abpr{i}) \label{P2}
\end{eqnarray*}

\noindent
and that a third condition
\begin{eqnarray*}
(P3) & & \hspace{0.7cm} \forall \gpr \in \gpairs, \gcost(\gpr) \ge \sum_{i=1}^k \abcost{i}(\abpr{i}) \hspace{0.7cm}
 \label{P3}
\end{eqnarray*}

\noindent
guarantees that $\haddtg$ is admissible and consistent.

Previous work has focused on defining abstraction and additivity for
specific ways of representing states and transition functions.
These are important contributions because ultimately one needs
computationally effective ways of defining the abstract state spaces,
abstraction mappings, and cost functions that our theory takes
as given.
The importance of our contribution is that it should make future
proofs of admissibility, consistency, and additivity easier, because
one will only need to show that a particular method for defining abstractions
satisfies the
three preceding conditions.  These are generally very simple conditions
to demonstrate, as we will now do for several methods for
defining abstractions and additivity that currently exist in the literature.


\subsubsection{Previous Definitions of Abstraction}
\label{previousAbstractions}

The use of abstraction to create heuristics began
in the late 1970s and was popularized
in Pearl's landmark book on heuristics \cite{PearlHeuristics}.
Two abstraction methods were identified at that time:
``relaxing" a state space definition by dropping operator
preconditions \cite{gaschnig79,guida79,PearlHeuristics,valtorta1984},
and ``homomorphic" abstractions \cite{AItheory,kibler}.
These early notions of abstraction
were unified and extended
by \citeA{mostow89}
and \citeA{machinediscovery}, producing
a formal definition that is the
same as ours in all important respects except for the
concept of ``residual cost" that we have introduced.\footnote{Prieditis's
definition allows an abstraction to expand the set of
goals.  This can be achieved in our definition by mapping non-goal states
in the original space to the same abstract state as the goal.}

Today's two most commonly used
abstraction methods are among the ones
implemented in Prieditis's Absolver II system \cite{machinediscovery}.
The first is ``domain abstraction", which was
independently introduced in the seminal work on
pattern databases \cite{PDB94,PDB98} and then generalized \cite{SARA2000}.
It assumes a state is represented by a set
of state variables, each of which has a set of possible values
called its domain.
An abstraction on states is defined by specifying a mapping from
the original domains to new, smaller domains.
For example, an 8-puzzle state is typically represented by 9 variables, one
for each location in the puzzle, each with the same domain of 9 elements,
one for each tile and one more for the blank.
A domain abstraction that maps all the elements representing the tiles
to the same new element (``don't care") and the blank to a
different element would produce the abstract space shown in
Figure \ref{fig-8puzzle}.
The reason this particular example satisfies property (P1)
is explained in Section \ref{DEFabstraction}.
In general, a domain abstraction will satisfy property (P1) as long as
the conditions that define when
state transitions occur ({\em e.g.}\ operator preconditions)
are guaranteed to be satisfied by the ``don't care" symbol whenever they are
satisfied by one or more of the domain elements that map to ``don't care".
Property (P2) follows immediately from the fact that all state
transitions in the original and abstract spaces have a primary cost of 1.

The other major type of abstraction used today, called ``drop" by
\citeA{machinediscovery}, was independently introduced
for abstracting planning domains represented by grounded (or propositional)
STRIPS operators \cite{planningPDB}.
In a STRIPS representation, a state is represented by the set of
logical atoms that are true in that state,
and
the directed edges between states are represented by a set of operators,
where each operator $a$ is described by three sets of atoms, $P(a)$, $A(a)$, and
$D(a)$.
$P(a)$ lists $a$'s preconditions: $a$ can be applied to state $t$ only if
all the atoms in $P(a)$ are true in $t$ ({\em i.e.}, $P(a) \subseteq  t$).
$A(a)$ and $D(a)$ specify the effects of operator $a$, with
$A(a)$ listing the atoms that become true when $a$ is applied (the ``add" list)
and
$D(a)$ listing the atoms that become false when $a$ is applied (the ``delete" list).
Hence if operator $a$ is applicable to state $t$, the state $u=a(t)$
it produces when applied to $t$ is the set of atoms $u= (t - D(a)) \cup A(a)$.

In this setting, Edelkamp defined an abstraction of a given state space
by specifying a subset of the atoms and 
restricting the abstract state descriptions and operator definitions
to include only atoms in the subset.
Suppose $V_i$ is the subset of the atoms underlying abstraction mapping
$\abmap{i} : \gspace \longrightarrow \abspace{i}$,
where $\gspace$ is the original state space and $\abspace{i}$ is the
abstract state space based on $V_i$.
Two states in $\gspace$ will be mapped to the same abstract state
if and only if they contain the same subset of atoms in $V_i$, {\em i.e.},
$\abmap{i}(t)=\abmap{i}(u)$ iff $t \cap V_i = u \cap V_i$.
This satisfies property (P1) because operator $a$ being applicable
to state $t$ ($P(a) \subseteq  t$) implies
abstract operator $a_i = \abmap{i}(a)$
is applicable to abstract state $t_i$
($P(a) \cap V_i \subseteq  t \cap V_i$)
and the resulting state $a(t) = (t - D(a)) \cup A(a)$
is mapped by $\abmap{i}$ to
$a_i(\abmap{i}(t))$ because set intersection distributes across set subtraction
and union
($V_i \cap ((t - D(a)) \cup A(a)) =
((V_i \cap t) -(V_i \cap D(a))) \cup (V_i \cap A(a))$).
Again, property (P2) follows immediately from the fact that all operators
in the original and abstract spaces have a primary cost of 1.


Recently, Helmert et al.~\citeyear{ICAPS2007Helmert} described a
more general approach to defining abstractions for planning based on
``transition graph abstractions".  A transition graph is
a directed graph in which the arcs have labels, and a transition graph
abstraction is a directed graph homomorphism that preserves
the labels.\footnote{``Homomorphism" here means the standard definition of
a digraph homomorphism (Definition \ref{def-digraph-homo}),
which permits non-surjectivity (as discussed
in Section \ref{DEFabstraction}), as opposed to
Helmert et al.'s definition of ``homomorphism", which does not allow
non-surjectivity.}
Hence, Helmert et al.'s method is a restricted version of our
definition of abstraction and therefore satisfies properties (P1) and (P2).
Helmert et al. make the following interesting observations that are true of
our more general definition of abstractions:
\begin{itemize}
\item the composition of two abstractions is an abstraction.
In other words, if
$\psi : \gspace \longrightarrow \mathbf{A} $
is an abstraction of $\gspace$ and
$\phi : \mathbf{A} \longrightarrow \mathbf{B}$
is an abstraction of $\mathbf{A}$, then
$(\phi \circ \psi ): \gspace \longrightarrow  \mathbf{B}$
is an abstraction of $\gspace$.  This property of abstractions was exploited
by Prieditis~\citeyear{machinediscovery}.

\item the ``product" $\abspace{1} \times \abspace{2}$
of two abstractions, $\abspace{1}$ and $\abspace{2}$,
of $\gspace$ is an
abstraction of $\gspace$, where the state space of the product
is the Cartesian product of the two abstract state spaces, and there
is an edge $\abpr{1 \times 2}$ in the product space
from state $(t_1,t_2)$ to state $(u_1,u_2)$
if there is an edge $\abpr{1}$ from $t_1$ to $u_1$ in $\abspace{1}$
and there is an edge $\abpr{2}$ from $t_2$ to $u_2$ in $\abspace{2}$.
The primary cost of $\abpr{1 \times 2}$ is the minimum of
$\abcost{1}(\abpr{1})$
and
$\abcost{2}(\abpr{2})$ and the residual cost of $\abpr{1 \times 2}$
is taken from the same space as the primary cost.
Because they are working with labelled edges Helmert et al. require
the edge connecting $t_1$ to $u_1$ to have the same label as the
edge connecting $t_2$ to $u_2$; this is called a ``synchronized" product
and is denoted $\abspace{1} \otimes \abspace{2}$ (refer
to Definition 6 defined by \citeA{ICAPS2007Helmert} for the exact definition of
synchronized product).

\end{itemize}

Figure \ref{fig-Helmert} shows the synchronized product, $B$,
of two abstractions, $\abspace{1}$ and $\abspace{2}$,
of the 3-state space $\gspace$ in which the edge labels are $a$ and $b$.
$\abspace{1}$ is derived from $\gspace$ by mapping states $s_1$ and $s_2$
to the same state ($s_{1,2}$), and
$\abspace{2}$ is derived from $\gspace$ by mapping states $s_2$ and $s_3$
to the same state ($s_{2,3}$).
Note that $B$ contains four states, more than the original space.
It is an abstraction of
$\gspace$ because the mapping of original
states $s_1$, $s_2$, and $s_3$ to states
$(s_{1,2},s_1)$ $(s_{1,2},s_{2,3})$ and $(s_3,s_{2,3})$, respectively,
satisfies property (P1), and property (P2) is satisfied automatically
because all edges have a cost of 1.
From this point of view the fourth state in $B$, $(s_3,s_1)$, is redundant
with state $(s_{1,2},s_1)$.
Nevertheless it is a distinct state in the product space.

\begin{figure}[!bht]
\centerline{
\includegraphics[height=6cm]{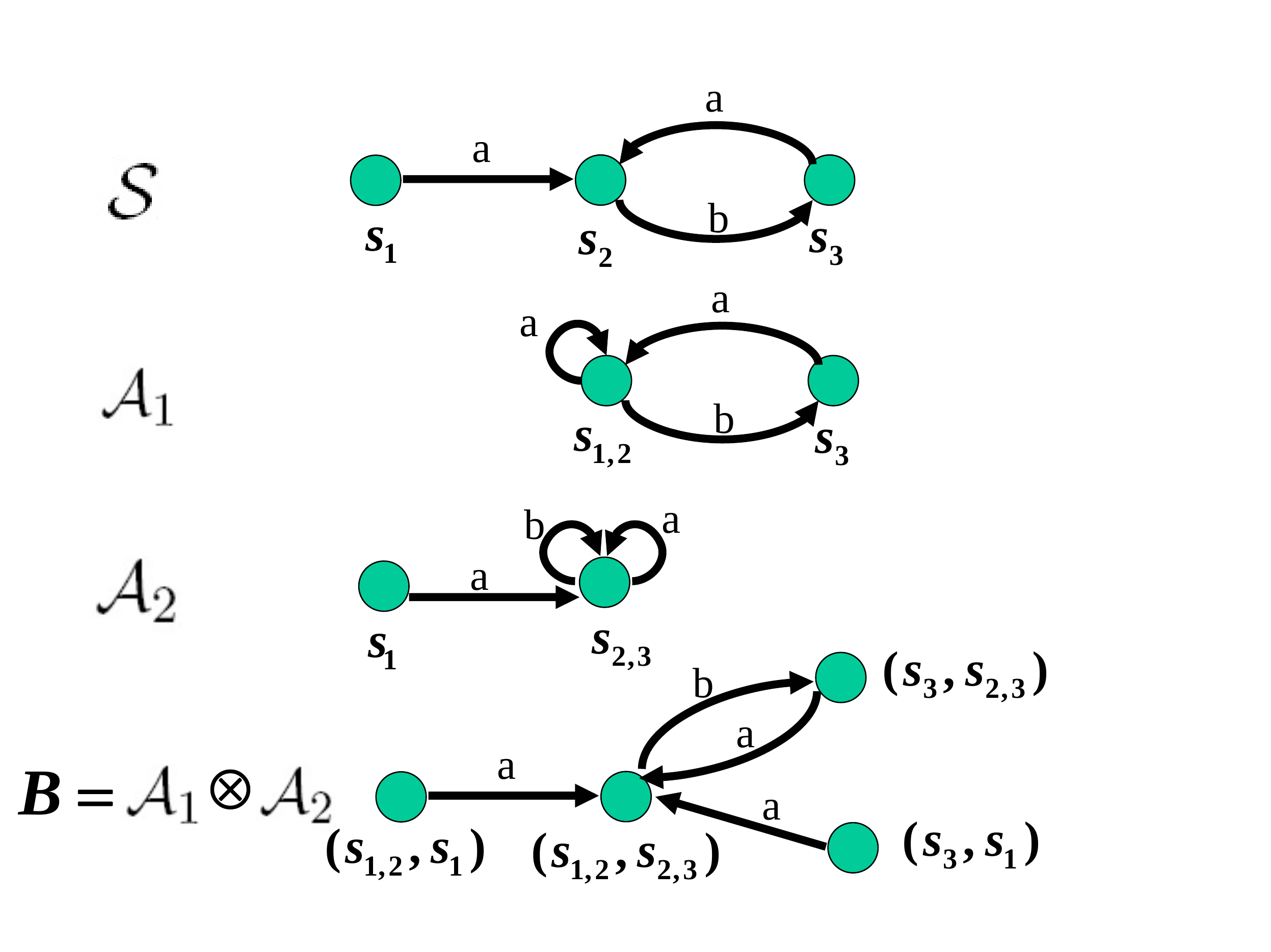} 
}
\caption{$\gspace$ is the original state space.
$\abspace{1}$ and $\abspace{2}$ are abstractions of $\gspace$.
$B=\abspace{1} \otimes \abspace{2}$
is the synchronized product of $\abspace{1}$ and $\abspace{2}$.
}
\label{fig-Helmert}
\end{figure}

Haslum et al.
\citeyear{AAAI2005Haslum} 
introduce a family of heuristics, called $h^m$ (for
any fixed $m \in \{1,2,...\}$),
that are based on abstraction, but are not covered by our definition
because the value of the heuristic for state $t$, $h^m(t)$,
is not defined as the distance from the abstraction of $t$ to the
abstract goal state.
Instead it takes advantage of a special monotonicity property of costs
in planning problems: the cost of achieving a subset of
the atoms defining the goal is a lower bound on the cost of achieving
the goal. When searching backwards from the goal to the start state,
as Haslum et al. do,
this allows an admissible heuristic to be defined in
the following recursive minimax fashion ($|t|$ denotes the number of atoms
in state $t$):


\input{haslum.tex}

\noindent
The first two lines of this definition are the standard method for
calculating the cost of a least-cost path. It is the third line that uses the
fact that the cost of achieving any subset of the atoms in $t$ is a lower bound
on the cost of achieving the entire set of atoms $t$.
The recursive calculation alternates between the min and max calculation
depending on the number of atoms in the state currently being
considered in the recursive calculation, and
is therefore different than a shortest path calculation or
taking the maximum of a set of shortest path calculations.


%% file: haslum.tex
\[
h^m(t) =  \left\{ \begin{array}{ll}
	0, & t \subseteq start\\
	{\displaystyle \min_{(s,t)\in\Pi } } C(s,t) + h^m(s), & |t| \le m\\
	{\displaystyle \max_{ s\subset t,  |s| \le m} } h^m(s), & |t| > m
	\end{array}
	\right.
\]

%% file: section_prev_add.tex
\subsubsection{Previous definitions of additive abstractions}

Prieditis \citeyear{machinediscovery} included a method (``Factor")
in his Absolver II system for creating additive abstractions, but
did not present any formal definitions or theory.

The first thorough discussion of additive abstractions is due to
Korf and Taylor~\citeyear{KorfTaylor1996}.
They observed that the
sliding tile puzzle's Manhattan Distance heuristic,
and several of its enhancements, were the sum of the distances
in a set of abstract spaces in which a small number of
tiles were ``distinguished".
As explained in Section \ref{additivePDB},
what allowed the abstract distances to be added and still be a lower bound
on distances in the original space is that
only the moves of the distinguished tiles counted towards the
abstract distance and no tile was distinguished in more than one abstraction.
This idea was later developed in a series of papers
\cite{disjointPDB,ADDPDB}, which extended its application
to other domains, such as the 4-peg Towers of Hanoi puzzle.

In the planning literature, the same idea was proposed by
Haslum et al.~\citeyear{AAAI2005Haslum}, who described it as
partitioning the operators into disjoint sets $B_1,...B_k$
and counting the cost of operators in set $B_i$ only in abstract space $A_i$.
The example they give is that in the Blocks World operators that
move block $i$ would all be in set $B_i$, effectively defining
a set of additive abstractions for the Blocks World exactly analogous
to the Korf and Taylor abstractions that define Manhattan Distance
for the sliding tile puzzle.

Edelkamp~\citeyear{planningPDB} took a different approach
to defining additive abstractions for STRIPS planning representations.
His method involves partitioning the atoms into disjoint sets $V_1,...V_k$
such that no operator changes atoms in more than one group.
If abstract space $A_i$ retains only the atoms in set $V_i$ then
the operators that do not affect atoms in $V_i$ will have no effect
at all in abstract space $A_i$ and will naturally have a cost of 0
in $A_i$.
Since no operator affects atoms in more than one group, no operator
has a non-zero cost in more than one abstract space and distances in
the abstract spaces can safely be added.
Haslum et al.~\citeyear{AAAI2007Haslum} extended this idea
to representations in which state variables could have multiple values.
In a subsequent paper Edelkamp \citeyear{EdelkampSymbolic}
remarks that if there
is no partitioning of atoms that induces a partitioning of the operators
as just described, additivity could be ``enforced" by
assigning an operator a cost of zero in all but one of the abstract
spaces---a return to the Korf and Taylor idea.

All the methods just described might be called ``all-or-nothing" methods
of defining abstract costs, because the cost of each edge $\gcost(\gpr)$
is fully assigned as the cost of
the corresponding abstract edge $\abcost{i}(\abpr{i})$
in one of the abstractions and the corresponding edges
in all the other abstractions are assigned a cost of zero.
Any such method obviously satisfies property (P3) and is therefore
additive.

Our theory of additivity does not require abstract methods to be defined
in an all-or-nothing manner, it allows  $\gcost(\gpr)$ to be divided
in any way whatsoever among the abstractions as long as property (P3)
is satisfied.
This possibility has been recognized in one recent publication
\cite{katz2007}, which did not report any experimental results.
This generalization is important because it eliminates the requirement
that operators must move only one ``tile" or change atoms/variables
in one ``group", and the related requirement that tiles/atoms
be distinguished/represented in exactly one of the abstract spaces.
This requirement restricted the application of previous methods
for defining additive abstractions,
precluding their application to state spaces such as
Rubik's Cube, the Pancake puzzle, and TopSpin.
As the following sections show, with our definition, additive
abstractions can be defined for any state space, including the three
just mentioned.

Finally, Helmert et al.~\citeyear{ICAPS2007Helmert}
showed that the synchronized product of additive abstractions
produces a heuristic $h_{sprod}$ that dominates $h_{add}$, in the
sense that $h_{sprod}(s) \ge h_{add}(s)$ for all states $s$.
This happens because the synchronized product forces the same
path to be used in all the abstract spaces, whereas the
calculation of each $\abminc{i}$ in $h_{add}$ can be based on
a different path.
The discussion of the negative results and infeasibility below
highlight the problems that can arise because each $\abminc{i}$
is calculated independently.


%% file: newApplicationsIntro.tex
\section{New Applications of Additive Abstractions}
\label{applications}

This section and the next section report
the results of applying the general definition
of additive abstraction given in the previous section to three
benchmark state spaces: TopSpin, the Pancake puzzle and Rubik's Cube.
A few additional experimental results
may be found in the previous paper by \citeA{additiveTheory07}.
In all our
experiments all edges in the original state spaces have a cost of 1
and we define $R_i(\abpr{i}) = 1 - \abcost{i}(\abpr{i})$,
its maximum permitted value when edges cost 1.
We use pattern databases to store the heuristic values.
The pre-processing time
required to compute the pattern databases is excluded from the times
reported in the results, because the PDB needs to be calculated only once and this overhead is
amortized over the solving of many problem instances.

\subsection{Methods for Defining Costs}

We will investigate two general methods for defining
the primary cost of an abstract state transition $\abcost{i}(\abpr{i})$,
which we
call ``cost-splitting" and ``location-based" costs.
To illustrate the generality of these methods we will define them
for the two most common ways of representing states---as a vector
of state variables, which is the method we implemented in our experiments,
and as a set of logical atoms as in the STRIPS representation for
planning problems.

In a state variable representation
a state $t$ is represented by a vector of $m$ state variables,
each having
its own domain of possible values $D_j$, {\em i.e.},
$t=(t(0),..., t(m-1))$, where $t(j) \in D_j$ is the value assigned
to the $j^{th}$ state variable in state $t$.
For example,
in puzzles such as the Pancake puzzle and the sliding tile puzzles,
there is typically one variable for each physical location in the puzzle,
and the value of $t(j)$ indicates which ``tile"
is in location $j$ in state $t$. In this case the domain for all the
variables is the same.
State space abstractions are defined by abstracting the domains.
In particular,  in this setting
domain abstraction $\abmap{i}$ will leave specific domain values unchanged
(the ``distinguished" values according to $\abmap{i}$)
and map all the rest to the same special value, ``don't care".
The abstract state corresponding to $t$ according to $\abmap{i}$ is
$t_i$=$(t_i(0),...,t_i(m-1))$ with $t_i(j) = \abmap{i}(t(j))$.
As in previous research with these state spaces
a set of abstractions is defined by partitioning
the domain values into disjoint sets $E_1,...,E_k$ with $E_i$ being
the set of distinguished values in abstraction $i$.
Note that the theory developed in the previous section
does not require the distinguished values in different abstractions
to be mutually exclusive; it allows a value to be distinguished in any
number of abstract spaces provided abstract costs are defined appropriately.

As mentioned previously, in a STRIPS representation a state is represented
by the set of logical atoms that are true in the state.
A state variable representation can be converted to a STRIPS representation
in a variety of ways, the simplest being to define an atom for each
possible variable-value combination. If state variable $j$ has value $v$
in the state variable representation of state $t$ then
the atom
$variable$-$j$-$has$-$value$-$v$
is true in the STRIPS representation of $t$.
The exact equivalent of domain abstraction can be achieved by
defining $V_i$, the set of atoms to be used
in abstraction $i$,
to be all the atoms
$variable$-$j$-$has$-$value$-$v$
in which $v \in E_i$, the set
of distinguished values for domain abstraction $i$.

\subsubsection{Cost-splitting}

In a state variable representation, the
cost-splitting method of defining primary costs works as follows.
A state transition $\pi$ that changes $b^{\pi}$ state variables 
has its cost, $C(\pi)$, split among the corresponding abstract
state transitions
$\pi_1, \dots ,\pi_k$
in proportion to the number of distinguished values they assign
to the variables,
{\em i.e.,} in abstraction $i$

\[
\abcost{i}(\abpr{i})= \frac{b^{\pi}_i*C(\pi)}{b^{\pi}}
\]

\noindent
if $\pi$ changes $b^{\pi}$ variables and $b^{\pi}_i$ of them are
assigned
distinguished values by $\pi_i$.\footnote{Because $\abpr{i}$ might
correspond to several edges in the the original space, each with a different
cost or moving a different set of tiles, the technically correct
definition is:
\[
\abcost{i}(\abpr{i})= \min_{\pi , \abmap{i}(\pi)=\abpr{i}} \frac{b^{\pi}_i*C(\pi)}{b^{\pi}}
\]
\label{minfootnote} }
For example,
the $3 \times 3 \times 3$ Rubik's cube is composed of
twenty little moveable ``cubies" and
each operator moves eight
cubies, four corner cubies and four edge cubies.
Hence $b^{\pi}=8$ for all state transitions $\pi$.
If a particular state transition moves three cubies that are
distinguished according to abstraction $\psi_i$, the corresponding
abstract state transition, $\abpr{i}$, would cost $\frac{3}{8}$.
Strictly speaking, we require abstract edge costs to be integers, so the
fractional edge costs produced by cost-splitting must be scaled
appropriately to become integers. Our implementation of cost-splitting
actually does this scaling but it will simplify our presentation of
cost-splitting to talk of the edge costs as if they were fractional.

If each domain value is distinguished in at most one
abstraction ({\em e.g.}\ if the abstractions are defined by partitioning
the domain values) cost-splitting produces additive abstractions, {\em
i.e.},
$\gcost(\gpr) \ge \sum_{i=1}^k
\abcost{i}(\abpr{i})$ for all $\gpr \in \gpairs$.
Because $\gcost(\gpr)$ is known to be an integer, $\hadd$ can be defined
to be the ceiling of the sum of the abstract distances,
$\lceil \hspace{0.1cm} \sum_{i=1}^k \abcost{i}(\abpr{i}) \rceil$,
instead of just the sum.

With a STRIPS representation, cost-splitting could be defined identically,
with $b^{\pi}$ being the number of atoms changed (added or deleted)
by operator $\pi$
in the original space and $b^{\pi}_i$ being the number of atoms
changed by the corresponding operator in abstraction $i$.

\subsubsection{Location-based Costs}

In a location-based cost definition for a state variable representation,
a state variable
$loc_{\pi}$ is associated with state transition $\pi$ and $\pi$'s
full cost $C(\pi)$ is assigned to abstract state transition
$\abpr{i}$ if $\abpr{i}$ changes the value of variable $loc_{\pi}$
to a value that is distinguished according to $\abmap{i}$.\footnote{As in
footnote \ref{minfootnote}, the technically correct definition has
$\min_{\pi, \abmap{i}(\pi)=\abpr{i}} C(\pi)$
instead of $C(\pi)$.}
Formally:

\[
\abcost{i}(\abpr{i})= \left\{
  \begin{array}{ll}
    C(\pi), & \mbox{if $\pi_i = (t^1_i,t^2_i)$, $t^1_i(loc^{\pi}) \ne t^2_i(loc^{\pi})$, and}\\
                            & \mbox{$t^2_i(loc^{\pi})$ is a distinguished value according to $\abmap{i}$.}
\\
    0, & \mbox{otherwise.}
  \end{array} \right.
\]

Instead of focusing on the value that is assigned to variable $loc_{\pi}$,
location-based costs
can be defined equally well on the value that variable $loc_{\pi}$
had before it was changed.
In either case, if each domain value is distinguished in at
most one abstraction location-based costs produce additive
abstractions.
The name ``location-based" is based on the typical representations
used for puzzles, in which there is a state variable for each
physical location in the puzzle.
For example, in Rubik's Cube one could choose the reference variables
to be the ones representing the two diagonally opposite corner locations
in the puzzle.
Note that each possible Rubik's cube operator changes
exactly one of these locations.
An abstract state transition would have a primary cost of 1 if the cubie
it moved into one of these locations was a distinguished cubie in its
abstraction, and a primary cost of 0 otherwise.

For a STRIPS representation of states, location-based costs can be
defined by choosing an atom $a$ in the $Add$ list for each operator $\pi$
and assigning the full cost $C(\pi)$ to abstraction $i$ if $a$ appears
in the $Add$ list of $\pi_i$. If atoms are partitioned so that each atom
appears in at most one abstraction, this method will define additive costs.

Although the cost-splitting and location-based methods for defining
costs can be applied to a wide range of state spaces, they are not
guaranteed to define heuristics that are superior to other
heuristics for a given state space. We determined experimentally
that heuristics based on cost-splitting substantially improve
performance for sufficiently large versions of TopSpin and that
heuristics based on location-based costs vastly improve the state of
the art for the 17-Pancake puzzle. In our experiments additive
heuristics did not improve the state of the art for Rubik's Cube.
The following subsections describe the positive results in detail.
The negative results are discussed in Section \ref{negative}.